\documentclass{article}

\usepackage{jmlr2e}

\usepackage{amsmath,amssymb,algorithm,algorithmicx,algpseudocode,amsthm,bm,bbm}
\usepackage{subfigure}
\usepackage{graphicx}
\usepackage{xspace}
\usepackage{paralist}
\usepackage{xcolor}

\usepackage{hyperref}
\usepackage{cleveref}
\usepackage{url}
\usepackage{notation}
\newcommand*\samethanks[1][\value{footnote}]{\footnotemark[#1]}

\title{Near-Optimal Representation Learning for Linear Bandits and Linear RL}

\author{\name Jiachen Hu\thanks{Equal contribution} \email NickH@pku.edu.cn \\
       \addr Key Laboratory of Machine Perception, MOE, School of EECS, Peking University
       \AND
       \name Xiaoyu Chen\samethanks \email cxy30@pku.edu.cn\\
       \addr Key Laboratory of Machine Perception, MOE, School of EECS, Peking University
       \AND
       \name Chi Jin \email chij@princeton.edu \\
       \addr Department of Electrical and Computer Engineering, Princeton University
       \AND
       \name Lihong Li \email llh@amazon.com \\
       \addr Amazon
       \AND
       \name Liwei Wang \email wanglw@cis.pku.edu.cn \\
       \addr Key Laboratory of Machine Perception, MOE, School of EECS, Peking University\\Center for Data Science, Peking University, Beijing Institute of Big Data Research}

\begin{document}

\maketitle








\begin{abstract}
This paper studies representation learning for multi-task linear bandits and multi-task episodic RL with linear value function approximation. We first consider the setting where we play $M$ linear bandits with dimension $d$ concurrently, and these bandits share a common $k$-dimensional linear representation so that $k\ll d$ and $k \ll M$. We propose a sample-efficient algorithm, MTLR-OFUL, which leverages the shared representation to achieve $\tilde{O}(M\sqrt{dkT} + d\sqrt{kMT} )$ regret, with $T$ being the number of total steps. Our regret significantly improves upon the baseline $\tilde{O}(Md\sqrt{T})$ achieved by solving each task independently. We further develop a lower bound that shows our regret is near-optimal when $d > M$. Furthermore, we extend the algorithm and analysis to multi-task episodic RL with linear value function approximation under low inherent Bellman error \citep{zanette2020learning}. To the best of our knowledge, this is the first theoretical result that characterizes the benefits of multi-task representation learning for exploration in RL with function approximation.
\end{abstract}

\section{Introduction}
\label{sec: introduction}

Multi-task representation learning is the problem of learning a common low-dimensional representation among multiple related tasks~\citep{caruana1997multitask}. This problem has become increasingly important 
in many applications such as natural language processing~\citep{ando2005framework,liu2019multi}, computer vision~\citep{li2014joint}, drug discovery~\citep{ramsundar2015massively}, and reinforcement learning~\citep{wilson2007multi,teh2017distral,d2019sharing}. In these cases, common information can be extracted from related tasks to improve data efficiency and accelerate learning.

While representation learning has achieved tremendous success in a variety of applications~\citep{bengio2013representation}, its theoretical understanding is still limited.  A widely accepted assumption in the literature is the existence of a common representation shared by different tasks.
For example, 
\citet{maurer2016benefit} proposed a general method to learn data representation in
multi-task supervised learning and learning-to-learn setting. \citet{du2020few} studied few-shot learning via representation learning with assumptions on a common representation among source and target tasks. \citet{tripuraneni2020provable} focused on the problem of multi-task linear regression with low-rank representation, 
and proposed algorithms with sharp statistical rates.
Inspired by the theoretical results in supervised learning, we take a step further to investigate provable benefits of representation learning for sequential decision making problems. 
First, we study the multi-task low-rank linear bandits problem, where $M$ tasks of $d$-dimensional (infinite-arm) linear bandits are concurrently learned for $T$ steps. The expected reward of arm $\bm{x}_i \in \dbR^d$ for task $i$ is  $\bm{\theta}_i^\top \bm{x}_i$, as determined by an unknown linear parameter $\bm{\theta}_i$.
To take advantage of the multi-task representation learning framework, we assume that $\bm{\theta}_i$'s lie in an unknown $k$-dimensional subspace of $\dbR^d$, where $k$ is much smaller compared to $d$ and $M$ \citep{yang2020provable}. 
The dependence among tasks makes it possible to achieve a regret bound better than solving each task independently.
Specifically, if the tasks are solved independently with standard algorithms such as OFUL~\citep{abbasi2011improved}, the total regret is $\tilde{O}(Md\sqrt{T})$.\footnote{$\tilde{O}$ hides the logarithmic factors.} By leveraging the common representation among tasks, we can achieve a better regret $\tilde{O}(M\sqrt{dkT} + d\sqrt{MkT})$. 
Our algorithm is also robust to the linear representation assumption when the model is misspecified. If the $k$-dimensional subspace approximates the rewards with error at most $\zeta$, our algorithm can still achieve regret $\tilde{O}(M\sqrt{dkT}+d\sqrt{kMT} + MT\sqrt{d}\zeta)$. 
Moreover, we prove a regret lower bound indicating that the regret of our algorithm is not improvable except for logarithmic factors in the regime $d > M$.

Compared with multi-task linear bandits, multi-task reinforcement learning is a more popular research topic with a long line of works in both theoretical side and empirical side \citep{taylor2009transfer,parisotto2015actor,liu2016decoding,teh2017distral,hessel2019multi,d2019sharing, arora2020provable}. We extend our algorithm for linear bandits to the multi-task episodic reinforcement learning with linear value function approximation under low inherent Bellman error \citep{zanette2020learning}. Assuming a low-rank linear representation across all the tasks, we propose a sample-efficient algorithm with regret $\tilde{O}(HM \sqrt{dkT} + Hd\sqrt{kMT} + HMT\sqrt{d} \caI)$ 
, where $k$ is the dimension of the 
low-rank representation, $d$ is the ambient dimension of state-action features, $M$ is the number of tasks, $H$ is the horizon, $T$ is the number of episodes, and $\caI$ denotes the inherent Bellman error. The regret significantly improves upon the baseline regret $\tilde{O}(HMd \sqrt{T} + HMT\sqrt{d} \caI)$ achieved by running ELEANOR algorithm \citep{zanette2020learning} for each task independently.
We also prove a regret lower bound $\Omega(Mk\sqrt{HT} + d\sqrt{HkMT} + HMT\sqrt{d} \caI)$. To the best of our knowledge, this is the first provably sample-efficient algorithm for exploration in multi-task linear RL.

\section{Preliminaries}
\label{sec: preliminaries}

\subsection{Multi-Task Linear Bandit}
\label{preliminaries:linear_bandits}
We study the problem of representation learning for linear bandits in which there are multiple tasks sharing common  low-dimensional features. Let $d$ be the ambient dimension and $k$ be the representation dimension. We play $M$ tasks concurrently for $T$ steps each. Each task $i \in [M]$ is associated with an unknown vector $\boldsymbol{\theta}_i \in \mathbb{R}^d$. In each step $t \in [T]$, the player chooses one action $\boldsymbol{x}_{t,i} \in \mathcal{A}_{t,i}$ for each task $i \in [M]$, and receives a batch of rewards $\{y_{t, i}\}_{i=1}^M$ afterwards, where $\mathcal{A}_{t,i}$ is the feasible action set (can even be chosen adversarially) for task $i$ at step $t$. 
The rewards received are determined by $y_{t,i} = \boldsymbol{\theta}_i^{\top}\boldsymbol{x}_{t,i} + \eta_{t,i}$, where the $\eta_{t,i}$ is the random noise. 

We use the total regret for $M$ tasks in $T$ steps to measure the performance of our algorithm, which is defined in the following way:
$$\operatorname{Reg}(T) \defeq \sum_{t=1}^{T}\sum_{i=1}^{M} \left(\left\langle \boldsymbol{x}_{t,i}^{\star}, \boldsymbol{\theta}_{i}\right\rangle-\left\langle \boldsymbol{x}_{t, i}, \boldsymbol{\theta}_{i}\right\rangle\right),$$
where $\boldsymbol{x}_{t,i}^{\star} = \argmax_{\boldsymbol{x} \in \mathcal{A}_{t,i}}\left\langle \boldsymbol{x}, \boldsymbol{\theta}_{i}\right\rangle$.

The main assumption is the existence of a common linear feature extractor.

\begin{assumption}
\label{assumptions:low_rank_bandits}
There exists a linear feature extractor $\boldsymbol{B} \in \mathbb{R}^{d\times k}$ and a set of $k$-dimensional coefficients $\{\boldsymbol{w}_i\}_{i=1}^{M}$ such that $\{\boldsymbol{\theta}_i\}_{i=1}^M$ satisfies $\boldsymbol{\theta}_i = \boldsymbol{B}\boldsymbol{w}_i$.
\end{assumption}

Define filtration $F_t$ to be the $\sigma$-algebra induced by $\sigma(\{\boldsymbol{x}_{1,i}\}_{i=1}^{M},\cdots, \{\boldsymbol{x}_{t+1,i}\}_{i=1}^{M},\{\eta_{1,i}\}_{i=1}^{M},\cdots,\{\eta_{t,i}\}_{i=1}^{M})$, then we have the following assumption.

\begin{assumption}
\label{assumptions:linear_bandits_regularity}
Following the standard regularity assumptions in linear bandits~\citep{abbasi2011improved,lattimore2020bandit}, we assume 
\begin{itemize}
    \item $\|\boldsymbol{\theta}_i\|_2 \leq 1, \forall i \in [M]$
    \item $\|\boldsymbol{x}\|_2 \leq 1, \forall \boldsymbol{x} \in \mathcal{A}_{t,i}, t \in [T], i \in [M]$
    \item $\eta_{t, i}$ is conditionally zero-mean $1$-sub-Gaussian random variable with regards to $F_{t-1}$.

\end{itemize}
\end{assumption}

For notation convenience, we use $\boldsymbol{X}_{t,i} = [\boldsymbol{x}_{1,i}, \boldsymbol{x}_{2,i}, \cdots, \boldsymbol{x}_{t,i} ]$ and $\boldsymbol{y}_{t,i} = [y_{1,i}, \cdots, y_{t,i}]^{\top}$ to denote the arms and the corresponding rewards collected for task $i \in [M]$ in the first $t$ steps, and we also use $\boldsymbol{\eta}_{t,i} = [\eta_{1,i},\eta_{2,i}, \cdots, \eta_{t,i}]^{\top}$ to denote the corresponding noise. We define $\boldsymbol{\Theta} \defeq [\boldsymbol{\theta}_1, \boldsymbol{\theta_2}, \cdots, \boldsymbol{\theta_M}]$ and $\boldsymbol{W} \defeq [\boldsymbol{w}_1, \boldsymbol{w}_2, \cdots, \boldsymbol{w}_M]$. For any positive definite matrix $\bm{A} \in \dbR^{d \times d}$, the Mahalanobis norm with regards to $\bm{A}$ is denoted by $\|\bm{x}\|_{\bm{A}} = \sqrt{\bm{x}^\top \bm{A} \bm{x}}$.



\subsection{Multi-Task Linear RL}

\label{preliminaries:linear_rl}
 We also study how this low-rank structure benefits the exploration problem with approximate linear value functions in multi-task episodic reinforcement learning. For reference convenience, we abbreviate our setting as multi-task LSVI setting, which is a natural extension of LSVI condition in the single-task setting ~\citep{zanette2020learning}.

Consider an undiscounted episodic MDP $\caM = (\caS, \caA, p, r, H)$ with state space $\caS$, action space $\caA$, and fixed horizon $H$. For any $h \in [H]$, any state $s_h \in \caS$ and action $a_h \in \caA$, the agent receives a reward $R_h(s_h, a_h)$ with mean $r_h(s_h, a_h)$, and transits to the next state $s_{h+1}$ according to the transition kernel $p_h\left(\cdot \mid s_h, a_h\right)$. The action value function for each state-action pair at step $h$ for some deterministic policy $\pi$ is defined as 
    $Q_h^\pi(s_h, a_h) \defeq r_h(s_h, a_h) + \dbE\left[\sum_{t=h+1}^H R_t(s_t, \pi_t(s_t)) \right]$
, and the state value function is defined as $V_h^\pi(s_h) = Q_h^\pi(s_h, \pi_h(s_h))$

Note that there always exists an optimal deterministic policy (under some regularity conditions) $\pi^*$ for which $V_h^{\pi^*}(s) = \max_\pi V_h^\pi(s) $ and $Q_h^{\pi^*}(s, a) = \max_\pi Q_h^{\pi}(s, a) $ for each $h \in [H]$. We denote $V_h^{\pi^*}$ and $Q_h^{\pi^*}$ by $V_h^*$ and $Q_h^*$ for short.

It's also convenient to define the Bellman optimality operator $\caT_h$ as 
$\caT_h(Q_{h+1})(s, a) \defeq r_h(s, a) + \dbE_{s' \sim p_h (\cdot \mid s, a)} \max_{a'} Q_{h+1}(s', a')$.



In the framework of single-task approximate linear value functions (see Section \ref{sec: linear_rl} for more discussions), we assume a feature map $\boldsymbol{\phi}: \caS \times \caA \to \dbR^{d}$ that maps each state-action pair to a $d$-dimensional vector. In case that $\caS$ is too large or continuous (e.g. in robotics), this feature map helps to reduce the problem scale from $|\caS| \times |\caA|$ to $d$. The value functions are the linear combinations of those feature maps, so we can define the function space at step $h \in [H]$ to be
$\caQ_h^\prime = \left\{Q_h(\boldsymbol{\theta}_h) \mid \boldsymbol{\theta}_h \in \Theta^\prime_h \right\}$ and $ \caV^\prime_h = \left\{V_h(\boldsymbol{\theta}_h) \mid \boldsymbol{\theta}_h \in \Theta^\prime_h \right\}$,
%
where $Q_h(\boldsymbol{\theta}_h)(s, a) \defeq \boldsymbol{\phi}(s, a)^\top \boldsymbol{\theta}_h$, and $V_h(\theta_h)(s) \defeq \max_a \boldsymbol{\phi}(s, a)^\top \boldsymbol{\theta}_h$. 

In order to find the optimal value function using value iteration with $\caQ_h$, we require that it is approximately close under $\caT_h$, as measured by the inherent Bellman error (or IBE for short). The IBE \citep{zanette2020learning} at step $h$ is defined as 
\begin{align}
\label{definitions:inherent_bellman_error}
\caI_h \defeq \!\!\! \sup_{Q_{h+1} \in \caQ_{h+1}} \inf_{Q_h \in \caQ_h} \sup_{s \in \caS, a \in \caA} \left|\left(Q_h - \caT_h(Q_{h+1})\right)(s, a)\right|.
\end{align}

In multi-task reinforcement learning, we have $M$ MDPs $\caM^1, \caM^2, ..., \caM^M$ (we use superscript $i$ to denote task $i$). Assume they share the same state space and action space, but have different rewards and transitions. 

To take advantage of the multi-task LSVI setting and low-rank representation learning, we define a joint function space for all the tasks as $\Theta
_h \defeq \{\left(\boldsymbol{B}_h\boldsymbol{w}^1_h, \boldsymbol{B}_h\boldsymbol{w}^2_h, \cdots, \boldsymbol{B}_h\boldsymbol{w}_h^M\right):  
\boldsymbol{B}_h \in \caO^{d \times k} , \boldsymbol{w}^i_h \in \caB^{k}, \boldsymbol{B}_h\boldsymbol{w}_h^i
\in \Theta_h^{i\prime}\}$,
%
%
where $\caO^{d \times k}$ is the collection of all orthonormal matrices in $\dbR^{d \times k}$. 


The induced function space is defined as 
\begin{align}
\caQ_h \defeq \{\left(Q_h^1\left(\boldsymbol{\theta}_h^1\right), Q_h^2\left(\boldsymbol{\theta}_h^2\right), \cdots, Q_h^M\left(\boldsymbol{\theta}_h^M\right)\right)
\mid \left(\boldsymbol{\theta}_h^1, \boldsymbol{\theta}_h^2, \cdots, \boldsymbol{\theta}_h^M\right) \in \Theta_h \}& \\
\caV_h \defeq \{\left(V_h^1\left(\boldsymbol{\theta}_h^1\right),V_h^2\left(\boldsymbol{\theta}_h^2\right), \cdots, V_h^M\left(\boldsymbol{\theta}_h^M\right)\right)
\mid \left(\boldsymbol{\theta}_h^1, \boldsymbol{\theta}_h^2, \cdots, \boldsymbol{\theta}_h^M\right) \in \Theta_h \} &  
\end{align}

The low-rank IBE at step $h$ for multi-task LSVI setting is a generalization of IBE (Eqn~\ref{definitions:inherent_bellman_error}) for the single-task setting, which is defined accordingly as
\begin{align}
\label{definitions:inherent_bellman_error_low_rank}
\caI_h^{\text{mul}} \defeq \sup_{\left\{Q_{h+1}^i\right\}_{i=1}^M \in \caQ_{h+1}} &\inf_{\left\{Q_{h}^i\right\}_{i=1}^M \in \caQ_{h}} \sup_{s \in \caS, a \in \caA, i \in [M]} \left|\left(Q_h^i - \caT_h^i(Q_{h+1}^i)\right)(s, a)\right|
\end{align}


\begin{assumption}
\label{assumptions:low_rank_small_ibe}
$\caI \defeq \sup_h \caI^{\text{mul}}_h$ is small with regards to the joint function space $\caQ_h$ for all $h$. 
\end{assumption}


When $\caI = 0$, Assumption \ref{assumptions:low_rank_small_ibe} can be regarded as a natural extension of Assumption \ref{assumptions:low_rank_bandits} in episodic RL. This is because there exists $\{\bar{\bm{\theta}}_h^{i*}\}_{i=1}^M \in \Theta_h$ such that $Q^{i*}_h = Q_h^i(\bar{\bm{\theta}}_h^{i*})$ for all $i \in [M]$ and $h \in [H]$ in the case $\caI = 0$. According to the definition of $\Theta_h$ we know that $\{\bar{\bm{\theta}}_h^{i*}\}_{i=1}^M$ also admit a low-rank property as Assumption \ref{assumptions:low_rank_bandits} indicates. When $\caI > 0$, then Assumption \ref{assumptions:low_rank_small_ibe} is an extension of misspecified multi-task linear bandits (discussed in Section \ref{sec: misspecifed linear bandits}) in episodic RL.

Define the filtration $\caF_{h, t}$ to be the $\sigma$-field induced by all the random variables up to step $h$ in episode $t$ (not include the rewards at step $h$ in episode $t$), then we have the following assumptions.

\begin{assumption}
\label{assumptions:linear_rl_regularity}
Following the parameter scale in \citep{zanette2020learning}, we assume
\begin{itemize} 
    \item $\left\|\boldsymbol{\phi}(s, a)\right\|_2 \leq 1, \forall (s, a) \in \caS \times \caA, h \in [H]$
    \item  $0 \leq Q^{\pi}_h(s,a) \leq 1, \forall (s, a) \in \caS \times \caA, h \in [H], \forall \pi$.
    \item There exists constant $D$ that for any $h \in [H]$ and any $\left\{\boldsymbol{\theta}_{h}^i\right\}_{i=1}^M \in \Theta_{h}$, it holds that $\|\boldsymbol{\theta}_{h}^i\|_2 \leq D, \forall i \in [M]$.
    \item For any fixed $\left\{Q_{h+1}^i\right\}_{i=1}^M \in \caQ_{h+1}$, the random noise $z_h^i(s, a) \defeq R_h^i(s, a) + \max_a Q_{h+1}^i\left(s', a\right) - \caT_h^i\left(Q_{h+1}^i\right)(s, a)$ is bounded in $[-1, 1]$ a.s., and is independent conditioned on $\caF_{h, t}$  for any $s \in \caS, a \in \caA, h \in [H], i \in [M]$, where the randomness is from reward $R$ and $s' \sim p_h\left(\cdot \mid s, a\right)$.
\end{itemize}
\end{assumption}


The first condition is a standard regularization condition for linear features. The second condition is on the scale of the problem. This scale of the exploration problem that the value function is bounded in $[0, 1]$ has also been studied in both tabular and linear setting \citep{zhang2020reinforcement, wang2020long,zanette2020learning}. The last two conditions are compatible with the scale of the problem. It's sufficient to assume the constant norm of $\boldsymbol{\theta}_{h}^i$ since the optimal value function is of the same scale. The last condition is standard in linear bandits \citep{abbasi2011improved, lattimore2020bandit} and RL~\citep{zanette2020learning}, and is automatically satisfied if $D = 1$.

The total regret of $M$ tasks in $T$ episodes is defined as
\begin{align}
\text{Reg}(T) \defeq \sum_{t=1}^T \sum_{i=1}^M \left(V_1^{i*} - V_1^{\pi_{t}^i} \right)\left(s_{1t}^i\right)
\end{align}
where $\pi_t^i$ is the policy used for task $i$ in episode $t$, and $s_{ht}^{i}$ denotes the state encountered at step $h$ in episode $t$ for task $i$. We assume $ M \geq 5, T \geq 5 $ throughout this paper.
\section{Related Work}
\label{sec: related work}

\paragraph{Multi-task Supervised Learning} The idea of multi-task representation learning at least dates back to \citet{caruana1997multitask,thrun1998learning,baxter2000model}. Empirically, representation learning has shown its great power in various domains. We refer readers to \citet{bengio2013representation} for a detailed review about empirical results. From the theoretical perspective, \citet{baxter2000model} performed the first theoretical analysis and gave sample complexity bounds using covering number. \citet{maurer2016benefit} considered the setting where all tasks are sampled from a certain distribution, and analysed the benefit of representation learning for both reducing the sample complexity of the target task. Following their results, \citet{du2020few} and \citet{tripuraneni2020provable} replaced the i.i.d assumption with a deterministic assumption on the data distribution and task diversity, and proposed efficient algorithms that can fully utilize all source data with better sample complexity. These results mainly focus on the statistical rate for multi-task supervised learning, and cannot tackle the exploration problem in bandits and RL.

\paragraph{Multi-task Bandit Learning} For multi-task linear bandits, the most related work is a recent paper by \citet{yang2020provable}. For linear bandits with infinite-action set, they firstly proposed an explore-then-exploit algorithm with regret $\tilde{O}(Mk\sqrt{T} + d^{1.5}k\sqrt{MT})$, which outperforms the naive approach with $\tilde{O}(Md\sqrt{T})$ regret in the regime where $M = \Omega(dk^2)$. Though their results are insightful, 
they require the action set for all tasks and all steps to be the same well-conditioned $d$-dimensional ellipsoids which cover all directions nicely with constant radius. Besides, they assume that the task parameters are diverse enough with $\boldsymbol{W}\boldsymbol{W}^{\top}$ well-conditioned, and the norm of $\boldsymbol{w}_i$ is lower bounded by a constant. These assumptions make the application of the theory rather restrictive to only a subset of linear bandit instances with benign structures. In contrast, our theory is more general since we do not assume the same and well-conditioned action set for different tasks and time steps, nor assume the benign properties of $\boldsymbol{w}_i$'s.

\paragraph{Multi-task RL} For multi-task reinforcement learning,  there is a long line of works from the empirical perspective \citep{taylor2009transfer,parisotto2015actor,liu2016decoding,teh2017distral,hessel2019multi}. From the theoretical perspective,
\citet{brunskill13mtrl} analyzed the sample complexity of multi-task RL in the tabular setting.
\citet{d2019sharing} showed that representation learning can improve the rate of approximate value iteration algorithm. \citet{arora2020provable} proved that representation learning can reduce the sample complexity of imitation learning. 

\paragraph{Bandits with Low Rank Structure} 
Low-rank representations have also been explored in single-task settings. \citet{jun2019bilinear} studied bilinear bandits with low rank representation. The mean reward in their setting is defined as the bilinear multiplication $\boldsymbol{x}^{\top} \boldsymbol{\Theta} \boldsymbol{y}$, where $\boldsymbol{x}$ and $\boldsymbol{y}$ are two actions selected at each step, and $\boldsymbol{\Theta}$ is an unknown parameter matrix with low rank. Their setting is further generalized by \citet{lu2020low}.  Furthermore, sparse linear bandits can be regarded as a simplified setting, where $\boldsymbol{B}$ is a binary matrix indicating the subset of relevant features in context $\boldsymbol{x}$~\citep{abbasi2012online,carpentier2012bandit,lattimore2015linear,hao2020high}.

\paragraph{Exploration in Bandits and RL} Our 
regret analysis is also related to exploration in single-task linear bandits and linear RL. Linear bandits have been extensively studied in recent years~\citep{auer2002using, dani2008stochastic,rusmevichientong2010linearly, abbasi2011improved, chu2011contextual,li2019nearly,li2019tight}. Our algorithm is most relevant to the seminal work of 
\citet{abbasi2011improved}, who applied self-normalized techniques to obtain near-optimal regret upper bounds. 
For single-task linear RL, recent years have witnessed a tremendous of works under different function approximation settings, including linear MDPs~\citep{yang2019sample, jin2020provably}, linear mixture MDPs~\citep{ayoub2020model,zhou2020nearly},  linear RL with low inherent Bellman error~\citep{zanette2020learning,zanette2020provably}, and MDPs with low Bellman-rank~\citep{jiang2017contextual}. Our multi-task setting is a natural extension of linear RL with low inherent Bellman error setting, which covers linear MDP setting as a special case~\citep{zanette2020learning}.
\section{Main Results for Linear Bandits}
\label{sec: linear_bandits}

In this section, we present our main results for multi-task linear bandits. 

\subsection{Construction of Confidence Sets}
\label{sec: construction of confidence sets}
A natural and successful method to design efficient algorithms for sequential decision making problem is the \textit{optimism in the face of uncertainty principle}. When applied to single-task linear bandits, the basic idea is to maintain a confidence set $\mathcal{C}_t$ for the parameter $\bm{\theta}$ based on history observations for each step $t \in [T]$. The algorithm chooses an optimistic estimation $\tilde{\boldsymbol{\theta}}_t = \operatorname{argmax}_{\boldsymbol{\theta} \in \mathcal{C}_{t}}\left(\max _{\boldsymbol{x} \in \mathcal{A}_{t}}\langle \boldsymbol{x}, \boldsymbol{\theta}\rangle\right)$ and then selects action $\boldsymbol{x}_t = \argmax_{\boldsymbol{x}_t \in \mathcal{A}_t} \langle \boldsymbol{x}, \tilde{\boldsymbol{\theta}}_t\rangle$, which maximizes the reward according to the estimation $\tilde{\boldsymbol{\theta}}_t$. In other words, the algorithm chooses the pair
$
\left(\boldsymbol{x}_{t}, \tilde{\boldsymbol{\theta}}_{t}\right)=\underset{(\boldsymbol{x}, \boldsymbol{\theta}) \in \mathcal{A}_{t} \times \mathcal{C}_{t}}{\operatorname{argmax}}\langle \boldsymbol{x}, \boldsymbol{\theta}\rangle
$.

For multi-task linear bandits, the main difference is that we need to tackle $M$ highly correlated tasks concurrently. To obtain tighter confidence bound, we maintain the confidence set $\mathcal{C}_t$ for $\boldsymbol{B}$ and $\{\boldsymbol{w}_i\}_{i=1}^{M}$, then choose the optimistic estimation $\tilde{\boldsymbol{\Theta}}_t$ for all tasks concurrently. To be more specific, the algorithm chooses an optimistic estimate $\tilde{\boldsymbol{\Theta}}_t = \argmax_{\boldsymbol{\Theta} \in \mathcal{C}_t} (\max_{\{x_i \in \mathcal{A}_{t,i}\}_{i=1}^{M}} \sum_{i=1}^{M}\left\langle \boldsymbol{x}_i, \boldsymbol{\theta}_{i}\right\rangle)$, and then selects action $\boldsymbol{x}_{t,i} = \argmax_{x_i \in \mathcal{A}_{t,i}} \left\langle \boldsymbol{x}_i, \tilde{\boldsymbol{\theta}}_{t,i}\right\rangle$ for each task $i \in [M]$. 

The main technical contribution is the construction of a tighter confidence set $\mathcal{C}_t$ for the estimation of $\boldsymbol{\Theta}$. At each step $t \in [T]$, we solve the following least-square problem based on the samples collected so far and obtain the minimizer $\hat{\boldsymbol{B}}_t$ and $\hat{\boldsymbol{W}}_t$:
\begin{align}
\label{eqn: optimization problem for linear bandits}
  \underset{\boldsymbol{B} \in \mathbb{R}^{d \times k}, {\boldsymbol{w}}_{1 .. M} \in \mathbb{R}^{k \times M}} 
{\arg \min } &\sum_{i=1}^{M} \left\|\boldsymbol{y}_{t-1,i}-\boldsymbol{X}_{t-1,i}^{\top} \boldsymbol{B} \boldsymbol{w}_{i}\right\|^{2}_2 \\
\mathrm{s.t.} \quad & \left\|\boldsymbol{B}\boldsymbol{w}_i\right\|_2 \leq 1, \forall i \in [M].
\end{align}

We maintain a high probability confidence set $\mathcal{C}_{t}$ for the unknown parameters $\boldsymbol{B}$ and $\{\boldsymbol{w}_i\}_{i=1}^{M}$. We calculate $\mathcal{C}_{t}$ in the following way:
\begin{align}
\label{eqn: confidence set for linear bandits}
\caC_t \defeq & \bigg\{\boldsymbol{\Theta} = \boldsymbol{B}\boldsymbol{W}:  \sum_{i=1}^{M}\left\|\hat{\boldsymbol{B}}_t \hat{\boldsymbol{w}}_{t,i}-\boldsymbol{B}\boldsymbol{w}_{i}\right\|^{2}_{\tilde{\boldsymbol{V}}_{t-1,i}(\lambda)} \leq L, \notag \\
&\phantom{=\;\;} \boldsymbol{B} \in \dbR^{d \times k}, \boldsymbol{w}_i \in \dbR^{k},  \left\|\boldsymbol{B}\boldsymbol{w}_i\right\|_2 \leq 1, \forall i \in [M]
\bigg\},
\end{align}
where 
$L = \tilde{O}(Mk + kd)$ (see Appendix \ref{sec: proof of lemma, confidence set for linear bandits} for the exact value)
and $\tilde{\boldsymbol{V}}_{t-1,i}(\lambda) = \boldsymbol{X}_{t-1,i}\boldsymbol{X}_{t-1,i}^{\top} + \lambda \boldsymbol{I}_d$. $\lambda$ is a hyperparameter used to ensure that $\tilde{\boldsymbol{V}}_{t-1,i}(\lambda)$ is always invertable, which can be set to $1$.  We can guarantee that $\boldsymbol{\Theta} \in \mathcal{C}_t$ for all $t \in [T]$ with high probability by the following lemma.

\begin{lemma}
\label{lemma: confidence set for linear bandits}
With probability at least $1-\delta$, for any step $t \in [T]$, suppose $\hat{\boldsymbol{\Theta}}_t = \hat{\boldsymbol{B}}_t \hat{\boldsymbol{W}}_t$ is the optimal solution of the least-square regression (Eqn~\ref{eqn: optimization problem for linear bandits}), the true parameter $\boldsymbol{\Theta} = \boldsymbol{B} \boldsymbol{W}$ is always contained in the confidence set $\mathcal{C}_t$, i.e.
\begin{align}
    \sum_{i=1}^{M}\left\|\hat{\boldsymbol{B}}_t \hat{\boldsymbol{w}}_{t,i}-\boldsymbol{B}\boldsymbol{w}_{i}\right\|^{2}_{\tilde{\boldsymbol{V}}_{t-1,i}(\lambda)} \leq L,
\end{align}
where $\tilde{\boldsymbol{V}}_{t-1,i}(\lambda) = \boldsymbol{X}_{t-1,i}\boldsymbol{X}_{t-1,i}^{\top} + \lambda \boldsymbol{I}_d$.
\end{lemma}


If we solve each tasks independently with standard single-task algorithms such as OFUL~\citep{abbasi2011improved}, it is not hard to realize that we can only obtain a confidence set with $\sum_{i=1}^{M}\|\hat{\boldsymbol{B}}_t \hat{\boldsymbol{w}}_{t,i}-\boldsymbol{B}\boldsymbol{w}_{i}\|^{2}_{\tilde{\boldsymbol{V}}_{t-1,i}(\lambda)} \leq L_1 = \tilde{O}(Md)$. Our confidence bound is much sharper compared with this naive bound, which explains the improvement in our final regret. Compared with \citet{yang2020provable}, we are not able to estimate $\boldsymbol{B}$ and $\boldsymbol{W}$ directly like their methods due to the more relaxed bandit setting. In our setting, the empirical design matrix $\tilde{\boldsymbol{V}}_{t-1,i}(\lambda)$ can be quite ill-conditioned if the action set at each step is chosen adversarially. Thus, we have to establish a tighter confidence set to improve the regret bound.

 We only sketch the main idea of the proof for Lemma~\ref{lemma: confidence set for linear bandits} and defer the detailed explanation to Appendix~\ref{sec: proof of lemma, confidence set for linear bandits}. Considering the non-trivial case where $d> 2k$, our main observation is that both $\boldsymbol{B}\boldsymbol{W}$ and $\hat{\boldsymbol{B}}_t \hat{\boldsymbol{W}}_t$ are low-rank matrix with rank upper bounded by $k$, which indicates that $\operatorname{rank}\left(\hat{\boldsymbol{B}}_t \hat{\boldsymbol{W}}_t-\boldsymbol{B}\boldsymbol{W}\right) \leq 2k$. Therefore, we can write $\hat{\boldsymbol{B}}_t \hat{\boldsymbol{W}}_t-\boldsymbol{B}\boldsymbol{W} = \boldsymbol{U}_t \boldsymbol{R}_t = [\boldsymbol{U}_t\boldsymbol{r}_{t,1}, \boldsymbol{U}_t\boldsymbol{r}_{t,2}, \cdots, \boldsymbol{U}_t\boldsymbol{r}_{t,M}]$, where $\boldsymbol{U}_t \in \mathbb{R}^{d \times 2k}$ is an orthonormal matrix and $\boldsymbol{R}_t \in \mathbb{R}^{2k\times M}$. Thus we have $$\boldsymbol{X}_{t-1,i}^{\top} \left(\hat{\boldsymbol{B}}_t \hat{\boldsymbol{w}}_{t,i} - \boldsymbol{B} \boldsymbol{w}_{i}\right) = \left(\boldsymbol{U}_t^{\top}\boldsymbol{X}_{t-1,i}\right)^{\top} \boldsymbol{R}_t.$$

This observation indicates that we can project the history actions $\boldsymbol{X}_{t-1,i}$ to a $2k$-dimensional space with $\boldsymbol{U}_t$, and take  $\boldsymbol{U}_t^{\top}\boldsymbol{X}_{t-1,i}$ as the $2k$-dimensional actions we have selected in the first $t-1$ steps. Following this idea, we connect the approximation error $\sum_{i=1}^{M}\left\|\hat{\boldsymbol{B}}_t \hat{\boldsymbol{w}}_{t,i}-\boldsymbol{B}\boldsymbol{w}_{i}\right\|^{2}_{\tilde{\boldsymbol{V}}_{t-1,i}(\lambda)}$ to the term $\sum_{i=1}^{M} \left\|\boldsymbol{\eta}_{t-1,i}^{\top} \left(\boldsymbol{U}_t^{\top}\boldsymbol{X}_{t-1,i}\right)^{\top} \right\|^2_{\boldsymbol{V}^{-1}_{t-1,i}(\lambda)}$, where $\boldsymbol{V}_{t-1,i}(\lambda) \stackrel{\text { def }}{=} \left( \boldsymbol{U}^{\top}_t \boldsymbol{X}_{t-1,i}\right)\left(\boldsymbol{U}_t^{\top}\boldsymbol{X}_{t-1,i} \right)^{\top}+\lambda \boldsymbol{I} $. We bound this term for the fixed $\boldsymbol{U}_t$ with the technique of self-normalized bound for vector-valued martingales~\citep{abbasi2011improved}, and then apply the $\epsilon$-net trick to cover all possible $\boldsymbol{U}_t$. This leads to an upper bound for $\sum_{i=1}^{M} \left\|\boldsymbol{\eta}_{t-1,i}^{\top} \boldsymbol{X}_{t-1,i}^{\top} \boldsymbol{U}_t\right\|^2_{\boldsymbol{V}^{-1}_{t-1,i}(\lambda)}$, and consequently helps to obtain the upper bound in Lemma~\ref{lemma: confidence set for linear bandits}.

\subsection{Algorithm and Regret}
\label{sec:alg and regret for linear bandits}

\begin{algorithm}[tbh]
\caption{Multi-Task Low-Rank OFUL}
\label{alg: multi-task OFUL}
  \begin{algorithmic}[1]
    \For { step $t = 1,2,\cdots, T$}
        \State Calculate the confidence interval $\mathcal{C}_t$ by Eqn~\ref{eqn: confidence set for linear bandits}
        \State $\tilde{\boldsymbol{\Theta}}_t, \bm{x}_{t,i} = \argmax_{ \boldsymbol{\Theta} \in \mathcal{C}_t, \bm{x}_i \in \mathcal{A}_{t,i}} \sum_{i=1}^{M}\left\langle \boldsymbol{x}_i, \boldsymbol{\theta}_{i}\right\rangle$
        \For{task $i = 1, 2,\cdots, M$}
            \State Play $\bm{x}_{t,i}$ for task $i$, and obtain the reward $y_{t,i}$
        \EndFor
    \EndFor
  \end{algorithmic}
\end{algorithm}

We describe our Multi-Task Low-Rank OFUL algorithm in Algorithm~\ref{alg: multi-task OFUL}. The following theorem states a bound on the regret of the algorithm.

\begin{theorem}
\label{thm: main theory for linear bandits}
Suppose Assumption~\ref{assumptions:low_rank_bandits} holds.  Then, with probability at least $1- \delta$, the regret of Algorithm~\ref{alg: multi-task OFUL} is bounded by
\begin{align}
\mathrm{Reg}(T) =\tilde{O}\left(M\sqrt{dkT} + d\sqrt{kMT} \right)
\end{align}
\end{theorem}

We defer the proof of Theorem~\ref{thm: main theory for linear bandits} to Appendix~\ref{sec: proof of main theorem for linear bandits}. The first term in the regret has linear dependence on $M$. This term characterizes the regret caused by learning the parameters $\boldsymbol{w}_i$ for each task. The second term has square root dependence on the number of total samples $MT$, which indicates the cost to learn the common representation with samples from $M$ tasks. By dividing the total regret by the number of tasks $M$, we know that the average regret for each task is $\tilde{O}(\sqrt{dkT} + d\sqrt{kT/M})$. Note that if we solve $M$ tasks with algorithms such as OFUL~\citep{abbasi2011improved} independently, the regret per task can be $\tilde{O}(d\sqrt{T})$. 
Our bound saves a factor of $\sqrt{d/k}$ compared with the naive method by leveraging the common representation features. We also show that when $d > M$ our regret bound is near optimal (see Theorem \ref{thm: lower bound for linear bandits}).

\subsection{Misspecified Multi-Task Linear Bandits}
\label{sec: misspecifed linear bandits}
For multi-task linear bandit problem, it is relatively unrealistic to assume a common feature extractor that can fit the reward functions of $M$ tasks exactly. A more natural situation is that the underlying reward functions are not exactly linear, but have some misspecifications. There are also relevant discussions on single-task linear bandits in recent works \citep{lattimore2020learning,zanette2020learning}. We first present a definition for the approximately linear bandit learning in multi-task setting.

\begin{assumption}
\label{assumption: misspecified linear bandits}
There exists a linear feature extractor $\boldsymbol{B} \in \mathbb{R}^{d\times k}$ and a set of linear coefficients $\{\boldsymbol{w}_i\}_{i=1}^{M}$ such that the expectation reward $\mathbb{E}[y_i|\boldsymbol{x}_i]$ for any action $\boldsymbol{x}_i \in \mathbb{R}^d$ satisfies  $\left|\mathbb{E}[y_i|\boldsymbol{x}_i] - \left\langle\boldsymbol{x}_i,\boldsymbol{B}\boldsymbol{w}_i\right\rangle\right| \leq \zeta$.
\end{assumption}

In general, an algorithm designed for a linear model could break down entirely if the underlying model is not linear. However, we find that our algorithm is in fact robust to small model misspecification if we set $L = \tilde{O}(Mk+kd+MT\zeta^2)$ (see Appendix~\ref{sec: proof for misspecified linear bandits} for the exact value). The following regret bound holds under Assumption~\ref{assumption: misspecified linear bandits} if we slightly modify the hyperparameter $L$ in the definition of confidence region $\mathcal{C}_t$.

\begin{theorem}
\label{thm: regret for misspecified linear bandits}
Under Assumption~\ref{assumption: misspecified linear bandits}, with probability at least $1- \delta$, the regret of Algorithm~\ref{alg: multi-task OFUL} is bounded by
\begin{align}
\mathrm{Reg}(T) = \tilde{O}\left(M\sqrt{dkT} + d\sqrt{kMT} + MT \sqrt{d }\zeta\right)
\end{align}
\end{theorem}

Theorem~\ref{thm: regret for misspecified linear bandits} is proved in Appendix~\ref{sec: proof for misspecified linear bandits}. Compared with Theorem~\ref{thm: main theory for linear bandits}, there is an additional term $\tilde{O}(MT \sqrt{d }\zeta)$ in the regret of Theorem~\ref{thm: regret for misspecified linear bandits}. This additional term is inevitably linear in $MT$ due to the intrinsic bias introduced by linear function approximation. Note that our algorithm can still enjoy good theoretical guarantees when $\zeta$ is sufficiently small.

\subsection{Lower Bound}
\label{sec: lower_bound for linear bandits}

In this subsection, we propose the regret lower bound for multi-task linear bandit problem under Assumption~\ref{assumption: misspecified linear bandits}.

\begin{theorem}
\label{thm: lower bound for linear bandits}
For any $k,M,d,T \in \mathbb{Z}^{+}$ with $k \leq d \leq T$ and $k \leq M$, and any learning algorithm $\mathcal{A}$, there exist a multi-task linear bandit instance that satisfies Assumption~\ref{assumption: misspecified linear bandits}, such that the regret of Algorithm~$\mathcal{A}$ is lower bounded by
$$\operatorname{Reg}(T) \geq \Omega\left(Mk\sqrt{T} + d\sqrt{kMT} + MT \sqrt{d }\zeta\right).$$
\end{theorem}

We defer the proof of Theorem~\ref{thm: lower bound for linear bandits} to Appendix~\ref{sec: proof of lower bound for linear bandits}. 
By setting $\zeta = 0$, Theorem~\ref{thm: lower bound for linear bandits} can be converted to the lower bound for multi-task linear bandit problem under Assumption~\ref{assumptions:low_rank_bandits}, which is $\Omega(Mk\sqrt{T} + d\sqrt{kMT} )$. These lower bounds match the upper bounds in Theorem~\ref{thm: main theory for linear bandits} and Theorem~\ref{thm: regret for misspecified linear bandits} in the regime where $d > M$ respectively. There is still a gap of $\sqrt{d/k}$ in the first part of the regret. For the upper bounds, the main difficulty to obtain $\tilde{O}(Mk\sqrt{T})$ regret in the first part comes from the estimation of $\boldsymbol{B}$. Since the action sets are not fixed and can be ill-conditioned, we cannot follow the explore-then-exploit framework and estimate $\boldsymbol{B}$ at the beginning. Besides, explore-then-exploit algorithms always suffer $\tilde{O}(T^{2/3})$ regret in the general linear bandits setting without further assumptions. Without estimating $\boldsymbol{B}$ beforehand with enough accuracy, the exploration in original $d$-dimensional space can be redundant since we cannot identify actions that have the similar $k$-dimensional representations before pulling them. We conjecture that our upper bound is tight and leave the gap as future work. 
\section{Main Results for Linear RL}
\label{sec: linear_rl}

We now show the main results for the multi-task episodic reinforcement learning under the assumption of low inherent Bellman error (i.e. the multi-task LSVI setting). 


\subsection{Multi-task LSVI Framework}
\label{linear_rl:MT-LSVI_framework}

In the exploration problems in RL where linear value function approximation is employed \citep{yang2019sample, jin2020provably,yang2020reinforcement}, LSVI-based algorithms are usually very effective when the linear value function space are \textit{close} under Bellman operator. For example, it is shown that a LSVI-based algorithm with additional bonus can solve the exploration challenge effectively in low-rank MDP \citep{jin2020provably}, where the function space $\caQ_h, \caQ_{h+1}$ are totally close under Bellman operator (i.e. any function $Q_{h+1}$ in $\caQ_{h+1}$ composed with Bellman operator $\caT_{h} \caQ_{h+1}$ belongs to $\caQ_{h}$). For the release of such strong assumptions, the inherent Bellman error for a MDP (Definition~\ref{definitions:inherent_bellman_error}) was proposed to measure how close is the function space under Bellman operator \citep{zanette2020learning}. We extend the definition of IBE to the multi-task LSVI setting (Definition~\ref{definitions:inherent_bellman_error_low_rank}), and show that our refined confidence set for the least square estimator can be applied to the low-rank multi-task LSVI setting, and gives an optimism-based algorithm with sharper regret bound compared to naively do exploration in each task independently.

\subsection{Algorithm}
\label{linear_rl:algorithm}

The MTLR-LSVI (Algorithm \ref{algorithm:linear_rl}) 
follows the LSVI-based \citep{jin2020provably, zanette2020learning} algorithms to build our (optimistic) estimator for the optimal value functions. To understand how this works for multi-task LSVI setting, we first take a glance at how LSVI-based algorithms work in single-task LSVI setting.

In traditional value iteration algorithms, we perform an approximate Bellman backup in episode $t$ for each step $h \in [H]$ on the estimator $Q_{h+1,t-1}$ constructed at the end of episode $t-1$, and find the best approximator for $\caT_h\left(Q_{h+1,t-1}\right)$ in function space $\caQ_h$. Since we assume linear function spaces, we can take the least-square solution of the empirical Bellman backup on $Q_{h+1,t-1}$ as the best approximator. 

In the multi-task framework, given an estimator $Q_{h+1}\left(\boldsymbol{\theta}_{h+1}^i\right)$ for each $i \in [M]$, to apply such least-square value iteration to our low-rank multi-task LSVI setting, we use the solution to the following constrained optimization problem
\begin{align}
\label{formula:linear_rl_least_square}
& \sum_{i=1}^M \sum_{j=1}^{t-1} \left(\left(\boldsymbol{\phi}_{hj}^i\right)^{\top} \boldsymbol{\theta}_h^i - R_{hj}^i -V_{h+1}^i\left(\boldsymbol{\theta}_{h+1}^i\right)\left(s_{h+1, j}^i\right)\right)^{2} \\
\text{s.t.} & \quad \boldsymbol{\theta}_h^1, \boldsymbol{\theta}_h^2, ..., \boldsymbol{\theta}_h^M \text{~lies in a $k$-dimensional subspace}
\end{align}

to approximate the Bellman update in the $t$-th episode, where $\boldsymbol{\phi}_{hj}^i = \boldsymbol{\phi}_h(s^i_{h j}, a^i_{hj})$ is the feature observed at step $h$ in episode $j$ for task $i$, and similarly $R_{hj}^i = R_h(s_{h j}^i, a_{h j}^i)$.

To guarantee the optimistic property of our estimator, we follow the global optimization procedure of \citet{zanette2020learning} which solves the following optimization problem in the $t$-th episode

\begin{definition}[Global Optimization Procedure]
\label{formula:linear_rl_global_optimization}

\begin{align}
\max_{\bar{\boldsymbol{\xi}}_h^i, \hat{\boldsymbol{\theta}}_h^i, \bar{\boldsymbol{\theta}}_h^i} & \sum_{i = 1}^M \max_{a^i} \left(\boldsymbol{\phi}(s_1^i, a^i)\right)^\top \bar{\boldsymbol{\theta}}_1^i \\
\text{s.t.} \quad & \left(\hat{\boldsymbol{\theta}}_h^1, ..., \hat{\boldsymbol{\theta}}_h^M\right) = \hat{\boldsymbol{B}}_h \begin{bmatrix} \hat{\boldsymbol{w}}_h^1 & \hat{\boldsymbol{w}}_h^2 & \cdots & \hat{\boldsymbol{w}}_h^M
\end{bmatrix} \notag \\
 & \qquad \qquad = \argmin_{\left\|\boldsymbol{B}_h \boldsymbol{w}_h^i\right\|_2 \leq D} \sum_{i=1}^M \sum_{j=1}^{t-1}  L(\boldsymbol{B}_h,\boldsymbol{w}_h^i)\\
& \bar{\boldsymbol{\theta}}_{h}^i =\hat{\boldsymbol{\theta}}_{h}^i + \bar{\boldsymbol{\xi}}_{h}^i; \quad \sum_{i = 1}^M \left\|\bar{\boldsymbol{\xi}}_{h}^i\right\|_{\tilde{\bm{V}}_{h t}^i(\lambda)}^2 \leq \alpha_{h t} \\
&\left(\bar{\boldsymbol{\theta}}_{h}^1, \bar{\boldsymbol{\theta}}_h^2, \cdots, \bar{\boldsymbol{\theta}}_h^M\right) \in \Theta_{h}
\end{align}
\end{definition}

where the empirical least-square loss $L(\boldsymbol{B}_h,\boldsymbol{w}_h^i)\defeq  ((\boldsymbol{\phi}_{hj}^i)^\top \boldsymbol{B}_h \boldsymbol{w}_h^i - R_{hj}^i - V_{h+1}^i(\bar{\boldsymbol{\theta}}_{h+1}^i)(s_{h+1, j}^i))^{2}$
, and
$
\tilde{\bm{V}}^i_{ht}(\lambda) \defeq \sum_{j=1}^{t-1} (\boldsymbol{\phi}_{hj}^i) (\boldsymbol{\phi}_{hj}^i)^\top + \lambda \bm{I}
$
is the regularized empirical linear design matrix for task $i$ in episode $t$.

\begin{algorithm}[!t]
\caption{Multi-Task Low-Rank LSVI}
\label{algorithm:linear_rl}
 \begin{algorithmic}[1]
    \State Input: low-rank parameter $k$, failure probability $\delta$, regularization $\lambda=1$, inherent Bellman error $\caI$
    \State Initialize $\tilde{\bm{V}}_{h1}=\lambda \boldsymbol{I}$ for $h \in [H]$
    \For{ episode $t = 1,2,\cdots $}
        \State Compute $\alpha_{ht}$ for $h \in [H]$. (see Lemma \ref{lemma:linear_rl_least_square_error})
        \State Solve the global optimization problem \ref{formula:linear_rl_global_optimization}
        \State Compute $\pi_{ht}^i(s) = \argmax_a \boldsymbol{\phi}(s, a)^\top \bar{\boldsymbol{\theta}}_{ht}^i$
        \State Execute $\pi^i_{ht}$ for task $i$ at step $h$
        \State Collect $\left\{s_{ht}^i, a_{ht}^i, r\left(s_{ht}^i, a_{ht}^i\right)\right\}$ for episode $t$.
    \EndFor
  \end{algorithmic}
\end{algorithm}


We have three types of variables in this global optimization problem, $\bar{\boldsymbol{\xi}}_h^i, \hat{\boldsymbol{\theta}}_h^i$, and $\bar{\boldsymbol{\theta}}_h^i$. Here $\bar{\boldsymbol{\theta}}_{h}^i$ denotes the estimator for $Q^{i *}_h$. We solve for the low-rank least-square solution of the approximate value iteration and denote the solution by $\hat{\boldsymbol{\theta}}_h^i$. Instead of adding the bonus term directly on $Q_{h}^i (\hat{\boldsymbol{\theta}}_h^i)$ to obtain an optimistic estimate of $Q^{i*}_h$ as in the tabular setting \citep{azar2017minimax, jin2018q} and linear MDP setting \citep{jin2020provably}, we use global variables $\bm{\bar{\xi}}_h^i$ to quantify the confidence bonus. This is because we cannot preserve the linear property of our estimator if we add the bonus directly, resulting in an exponential propagation of error. However, by using $\bar{\boldsymbol{\xi}}_h^i$ we can construct a linear estimator $Q_{h}^i \left(\bar{\boldsymbol{\theta}}_h^i\right)$ and obtain much smaller regret. A drawback of this global optimization technique is that we can only obtain an optimistic estimator at step 1, since values in different states and steps are possibly negatively correlated.


\subsection{Regret Bound}
\label{linear_rl:regret_bound}
\begin{theorem}
\label{theorem:linear_rl_regret_bound}
Under Assumption \ref{assumptions:low_rank_small_ibe} and \ref{assumptions:linear_rl_regularity}, with probability $1 - \delta$ the regret after $T$ episodes is bounded by
\begin{align}
\operatorname{Reg}(T) = \tilde{O} \left(H M \sqrt{dkT} + Hd\sqrt{kMT} + HMT\sqrt{d} \caI\right)
\end{align}
\end{theorem}


Compared to naively executing single-task linear RL algorithms (e.g. the ELEANOR algorithm) on each task without information-sharing, which incurs regret $\tilde{O} (HM d\sqrt{T} + HMT\sqrt{d} \caI)$, our regret bound is smaller by a factor of approximately $\sqrt{d/k}$ in our setting where $k \ll d$ and $k \ll M$.

We give a brief explanation on how we improve the regret bound and defer the full analysis to appendix \ref{sec: omiited proof in linea_rl}. We start with the decomposition of the regret. Let $\bar{Q}_{ht}^i$($\bar{V}_{ht}^i$) be the solution of the problem in definition \ref{formula:linear_rl_global_optimization} in episode $t$, then
\begin{align}
& \text{Reg}(T) = \sum_{t=1}^T \sum_{i=1}^M \left(V_1^{i*} - \bar{V}_{1t}^i +  \bar{V}_{1t}^i - V_1^{\pi_{t}^i} \right)\left(s_{1t}^i\right) \\
\label{formula:linear_rl_sketch_optimism}
& \leq HMT\caI \qquad \text{(by Lemma \ref{lemma:linear_rl_optimism})}\\
\label{formula:linear_rl_sketch_bellman_error}
& + \sum_{t=1}^T \sum_{h=1}^H \sum_{i=1}^M \left( \left|\bar{Q}_{ht}^i (s, a) - \caT_h^i \bar{Q}_{h+1,t}^i (s, a)\right| + \zeta_{ht}^i\right).
\end{align}

In (\ref{formula:linear_rl_sketch_optimism}) we use the optimistic property of $\bar{V}_{1t}^i$. In  (\ref{formula:linear_rl_sketch_bellman_error}),   $\zeta_{ht}^i$ is a martingale difference (defined in section \ref{appendix_linear_rl:regret_bound}) with regards to $\caF_{h,t}$, and the dominate term (the first term) is the Bellman error of $\bar{Q}_{ht}^i$.


For any $\{Q_{h+1}^i\}_{i=1}^M \in \caQ_{h+1}$, we can find a group of vectors $\{\dot{\boldsymbol{\theta}}_{h}^i(Q_{h+1}^i) \}_{i=1}^M \in \Theta_h$ that satisfy $\Delta_h^i \left(Q_{h+1}^i\right)(s, a) \defeq \caT_h^i \left(Q_{h+1}^i\right)(s, a) - \boldsymbol{\phi}(s, a)^\top \dot{\boldsymbol{\theta}}_{h}^i\left(Q_{h+1}^i\right)$ and the approximation error $\left\|\Delta_h^i \left(Q_{h+1}^i\right)\right\|_{\infty} \leq \caI$ is small for each $i \in [M]$. By definition, $\dot{\boldsymbol{\theta}}_{h}^i\left(Q_{h+1}^i\right)$ is actually the best approximator of $\caT_h^i \left(Q_{h+1}^i\right)$ in the function class $\caQ_h$. Since our algorithm is based on least-square value iteration, a key step is to bound the error of estimating $\dot{\boldsymbol{\theta}}_{h}^i(\bar{Q}_{h+1,t}^i)$ ($\dot{\boldsymbol{\theta}}_{h}^i$ for short). In the global optimization procedure, we use $\hat{\boldsymbol{\theta}}_h^i$ to approximate the empirical Bellman backup. In Lemma \ref{lemma:linear_rl_least_square_error} we show 
\begin{align}
\label{formula:linear_rl_sketch_least_square_error}
\sum_{i=1}^{M}\left\|\hat{\boldsymbol{\theta}}_h^i - \dot{\boldsymbol{\theta}}_h^i\right\|^{2}_{\tilde{\boldsymbol{V}}_{ht}^i(\lambda)} = \tilde{O}\left(Mk + kd + MT\caI^2\right)
\end{align}

This is the key step leading to improved regret bound. If we solve each task independently without information sharing, we can only bound the least square error in (\ref{formula:linear_rl_sketch_least_square_error}) as $\tilde{O}(Md+MT\caI^2)$. Our bound is much more sharper since $k \ll d$ and $k \ll M$.

Using the least square error in (\ref{formula:linear_rl_sketch_least_square_error}), we can show that the dominate term in (\ref{formula:linear_rl_sketch_bellman_error}) is bounded by (see Lemma \ref{lemma:linear_rl_bellman_error} and section \ref{appendix_linear_rl:regret_bound})
\begin{align}
& \sum_{i=1}^M \left|\bar{Q}_{ht}^i (s, a) - \caT_h^i \bar{Q}_{h+1,t}^i (s, a)\right| \leq M\caI +  \tilde{O}\left(\sqrt{Mk + kd + MT\caI^2}\right) \cdot \sqrt{\sum_{i=1}^M \left\|
\boldsymbol{\phi}(s_{ht}^i, a_{ht}^i)\right\|_{\tilde{\boldsymbol{V}}_{ht}^i(\lambda)^{-1}}^2} 
\end{align}

\citet[Lemma 11]{abbasi2011improved} states that $\sum_{t=1}^T \left\|
\boldsymbol{\phi}(s_{ht}^i, a_{ht}^i)\right\|_{\tilde{\boldsymbol{V}}_{ht}^i(\lambda)^{-1}}^2 = \tilde{O}(d)$ for any $h$ and $i$, so we can finally bound the regret as
\begin{align*}
\text{Reg}(T) & = \tilde{O}\left(HMT\caI + H\sqrt{Mk + kd + MT\caI^2} \cdot \sqrt{MTd}\right) \notag \\
& = \tilde{O}\left(H M \sqrt{dkT} + Hd\sqrt{kMT} + HMT\sqrt{d} \caI\right)
\end{align*}
where the first equality is by Cauchy-Schwarz.

\subsection{Lower Bound}
\label{linear_rl:lower_bound}

This subsection presents the lower bound for multi-task reinforcement learning with low inherent Bellman error. Our lower bound is derived from the lower bound in the single-task setting. As a byproduct, we also derive a lower bound for misspecified linear RL in the single-task setting. We defer the proof of Theorem~\ref{theorem:linear_rl_lower_bound} to Appendix~\ref{sec: proof of the lower bound for rl}.

\begin{theorem}
\label{theorem:linear_rl_lower_bound}
For our construction in appendix \ref{sec: proof of the lower bound for rl}, the expected regret of any algorithm where $d,k,H \geq 10, |\mathcal{A}| \geq 3,M \geq k,  T= \Omega(d^2H),\mathcal{I} \leq 1 / 4H$ is


$$
\Omega \left(Mk\sqrt{HT} + d\sqrt{HkMT} +  HMT \sqrt{d} \mathcal{I}\right)
$$

\end{theorem}

Careful readers may find that there is a gap of $\sqrt{H}$ in the first two terms between the upper bound and the lower bound. This gap is because the confidence set used in the algorithm is intrinsically ``Hoeffding-type''. Using a ``Bernstein-type'' confidence set can potentially improve the upper bound by a factor of $\sqrt{H}$. This ``Bernstein'' technique has been well exploited in many previous results for single-task RL~\citep{azar2017minimax,jin2018q, zhou2020nearly}. Since 
our focus is mainly on the benefits of multi-task representation learning, we don't apply this technique for the clarity of the analysis. If we ignore this gap in the dependence on $H$, our upper bound matches this lower bound in the regime where $d \geq M$.
\section{Conclusion}
\label{sec: conclusion}

In this paper, we study provably sample-efficient representation learning for multi-task linear bandits and linear RL. For linear bandits, we propose an algorithm called MTLR-OFUL, which obtains near-optimal regret in the regime where $d \geq M$. We then extend our algorithms to multi-task RL setting, and propose a sample-efficient algorithm, MTLR-LSVI.

There are two directions for future investigation. First, our algorithms are statistically sample-efficient, but a computationally efficient implementation is still unknown, although we conjecture our MTLR-OFUL algorithm is computationally efficient. How to design both computationally and statistically efficient algorithms in our multi-task setting is an interesting problem for future research. Second, there remains a gap of $\sqrt{d/k}$ between regret upper and lower bounds (in the first term).
We conjecture that our lower bound is not minimax optimal and hope to address this problem in the future work. 

\bibliography{ref.bib}

\clearpage
\newpage

\clearpage
\setcounter{section}{0}
\appendix
\renewcommand{\appendixname}{Appendix~\Alph{section}}
\section*{Appendices}
\addcontentsline{toc}{section}{Appendices}
\renewcommand{\thesubsection}{\Alph{subsection}}

\subsection{Omitted Proof in Section~\ref{sec: linear_bandits}}
\label{sec: omitted proof for linear bandits}

\subsubsection{Proof of Lemma~\ref{lemma: confidence set for linear bandits}}
\label{sec: proof of lemma, confidence set for linear bandits}

\begin{proof}
By the optimality of $\hat{\boldsymbol{B}}_t$ and $\hat{\boldsymbol{W}}_t = [\hat{\boldsymbol{w}}_{t,1}, \cdots, \hat{\boldsymbol{w}}_{t,M}]$, we know that $\sum_{i=1}^{M} \left\|\boldsymbol{y}_{t-1,i}-\boldsymbol{X}_{t-1,i}^{\top} \hat{\boldsymbol{B}}_t \hat{\boldsymbol{w}}_{t,i}\right\|^{2}_2 \leq \sum_{i=1}^{M} \left\|\boldsymbol{y}_{t-1,i}-\boldsymbol{X}_{t-1,i}^{\top} \boldsymbol{B} \boldsymbol{w}_{i}\right\|^{2}_2$. Since $\boldsymbol{y}_{t-1,i} = \boldsymbol{X}_{t-1,i}^{\top} \boldsymbol{B} \boldsymbol{w}_{i} + \boldsymbol{\eta}_{t-1,i}$, we have
\begin{align}
    \label{eqn: optimality of hTheta}
    \sum_{i=1}^{M} \left\|\boldsymbol{X}_{t-1,i}^{\top} \left(\hat{\boldsymbol{B}}_t \hat{\boldsymbol{w}}_{t,i} - \boldsymbol{B} \boldsymbol{w}_{i}\right)\right\|^2_2 \leq 2\sum_{i=1}^{M} \boldsymbol{\eta}_{t-1,i}^{\top} \boldsymbol{X}_{t-1,i}^{\top} \left(\hat{\boldsymbol{B}}_t \hat{\boldsymbol{w}}_{t,i} - \boldsymbol{B} \boldsymbol{w}_{i}\right).
\end{align}

We firstly analyse the non-trivial setting where $d \geq 2k$. Note that both $\boldsymbol{\Theta} = \boldsymbol{B}\boldsymbol{W}$ and $\hat{\boldsymbol{\Theta}}_t = \hat{\boldsymbol{B}}_t \hat{\boldsymbol{W}}_t$ are low-rank matrix with rank upper bounded by $k$, which indicates that $\operatorname{rank}\left(\hat{\boldsymbol{\Theta}}_t-\boldsymbol{\Theta}\right) \leq 2k$. In that case, we can write $\hat{\boldsymbol{\Theta}}_t-\boldsymbol{\Theta} = \boldsymbol{U}_t \boldsymbol{R}_t = [\boldsymbol{U}_t\boldsymbol{r}_{t,1}, \boldsymbol{U}_t\boldsymbol{r}_{t,2}, \cdots, \boldsymbol{U}_t\boldsymbol{r}_{t,M}]$, where $\boldsymbol{U}_t \in \mathbb{R}^{d \times 2k}$ is an orthonormal matrix with $\|\boldsymbol{U}_t\|_F = \sqrt{2k}$, and $\boldsymbol{R}_t \in \mathbb{R}^{2k\times M}$ satisfies $\|\boldsymbol{r}_{t,i}\|_2 \leq \sqrt{k}$. In other words, we can write $\hat{\boldsymbol{B}}_t \hat{\boldsymbol{w}}_{t,i} - \boldsymbol{B} \boldsymbol{w}_{i} = \boldsymbol{U}_t \boldsymbol{r}_{t,i}$ for certain $\boldsymbol{U}_t$ and $\boldsymbol{r}_{t,i}$.

Define $\boldsymbol{V}_{t-1,i}(\lambda) \stackrel{\text { def }}{=} \left( \boldsymbol{U}^{\top}_t \boldsymbol{X}_{t-1,i}\right)\left(\boldsymbol{U}_t^{\top}\boldsymbol{X}_{t-1,i} \right)^{\top}+\lambda \boldsymbol{I} $.  We have:
\begin{align}
    \lefteqn{\sum_{i=1}^{M}\left\|\hat{\boldsymbol{B}}_t \hat{\boldsymbol{w}}_{t,i}-\boldsymbol{B}\boldsymbol{w}_{i}\right\|^{2}_{\tilde{\boldsymbol{V}}_{t-1,i}(\lambda)}} \label{eqn: hTheta - Theta decomposition 0} \\
    \label{eqn: hTheta - Theta decomposition 1}
     = &  \sum_{i=1}^{M} \left\|\boldsymbol{X}_{t-1,i}^{\top} \left(\hat{\boldsymbol{B}}_t \hat{\boldsymbol{w}}_{t,i} - \boldsymbol{B} \boldsymbol{w}_{i}\right)\right\|^2_2 + \sum_{i=1}^{M} \lambda \left\|\hat{\boldsymbol{B}}_t \hat{\boldsymbol{w}}_{t,i}-\boldsymbol{B}\boldsymbol{w}_{i}\right\|^{2}_2 \\
     \label{eqn: hTheta - Theta decomposition 2}
     \leq & 2\sum_{i=1}^{M} \boldsymbol{\eta}_{t-1,i}^{\top} \boldsymbol{X}_{t-1,i}^{\top} \left(\hat{\boldsymbol{B}}_t \hat{\boldsymbol{w}}_{t,i} - \boldsymbol{B} \boldsymbol{w}_{i}\right) + 4M \lambda \\
     \label{eqn: hTheta - Theta decomposition 3}
     = & 2\sum_{i=1}^{M} \boldsymbol{\eta}_{t-1,i}^{\top} \boldsymbol{X}_{t-1,i}^{\top} \boldsymbol{U}_t \boldsymbol{r}_{t,i} + 4M \lambda \\
     \label{eqn: hTheta - Theta decomposition 4}
     \leq &  2\sum_{i=1}^{M} \left\|\boldsymbol{\eta}_{t-1,i}^{\top} \boldsymbol{X}_{t-1,i}^{\top} \boldsymbol{U}_t\right\|_{\boldsymbol{V}^{-1}_{t-1,i}(\lambda)} \left\|\boldsymbol{r}_{t,i}\right\|_{\boldsymbol{V}_{t-1,i}(\lambda)} + 4M\lambda \\
     \label{eqn: hTheta - Theta decomposition 5}
     \leq & 2\sqrt{\sum_{i=1}^{M} \left\|\boldsymbol{\eta}_{t-1,i}^{\top} \boldsymbol{X}_{t-1,i}^{\top} \boldsymbol{U}_t\right\|^2_{\boldsymbol{V}^{-1}_{t-1,i}(\lambda)}} \sqrt{\sum_{i=1}^{M}\left\|\boldsymbol{r}_{t,i}\right\|^2_{\boldsymbol{V}_{t-1,i}(\lambda)}} + 4M\lambda \\
     \label{eqn: hTheta - Theta decomposition 6}
     = & 2\sqrt{\sum_{i=1}^{M} \left\|\boldsymbol{\eta}_{t-1,i}^{\top} \boldsymbol{X}_{t-1,i}^{\top} \boldsymbol{U}_t\right\|^2_{\boldsymbol{V}^{-1}_{t-1,i}(\lambda)}} \sqrt{\sum_{i=1}^{M}\left\|\hat{\boldsymbol{B}}_t \hat{\boldsymbol{w}}_{t,i}-\boldsymbol{B}\boldsymbol{w}_{i}\right\|^{2}_{\tilde{\boldsymbol{V}}_{t-1,i}(\lambda)}} + 4M\lambda
\end{align}

Eqn~\ref{eqn: hTheta - Theta decomposition 2} is due to Eqn~\ref{eqn: optimality of hTheta}, $\left\|\hat{\boldsymbol{B}}_t \hat{\boldsymbol{w}}_{t,i}\right\| \leq 1$ and $\left\|\boldsymbol{B} \boldsymbol{w}_{i}\right\| \leq 1$. Eqn~\ref{eqn: hTheta - Theta decomposition 5} is due to Cauchy-Schwarz inequality. Eqn~\ref{eqn: hTheta - Theta decomposition 6} is from 
$$\sum_{i=1}^{M}\left\|\hat{\boldsymbol{B}}_t \hat{\boldsymbol{w}}_{t,i}-\boldsymbol{B}\boldsymbol{w}_{i}\right\|^{2}_{\tilde{\boldsymbol{V}}_{t-1,i}(\lambda)} = \sum_{i = 1}^M \left\|\boldsymbol{U}_{t}\boldsymbol{r}_{t,i}\right\|^2_{\tilde{\boldsymbol{V}}_{t-1,i}(\lambda)} = \sum_{i = 1}^M \left\|\boldsymbol{r}_{t,i}\right\|^2_{\boldsymbol{U}_t^\top \tilde{\boldsymbol{V}}_{t-1, i}(\lambda)  \boldsymbol{U}_t} = \sum_{i = 1}^M \left\|\boldsymbol{r}_{t,i}\right\|^2_{\boldsymbol{V}_{t-1, i}(\lambda)}.$$

The main problem is how to bound $\left\|\boldsymbol{\eta}_{t-1,i}^{\top} \boldsymbol{X}_{t-1,i}^{\top} \boldsymbol{U}_t\right\|_{\boldsymbol{V}^{-1}_{t-1,i}(\lambda)} = \left\|\sum_{n=1}^{t-1}\eta_{n,i}  \boldsymbol{U}^{\top}_t x_{n,i}\right\|_{\boldsymbol{V}^{-1}_{t-1,i}(\lambda)}$. Note that for a fixed $\boldsymbol{U}_t = \bar{\boldsymbol{U}}$, we can regard $\bar{\boldsymbol{U}}^{\top} \boldsymbol{x}_{n,i} \in \mathbb{R}^{k}$ as the corresponding ``action'' chosen in step $t$. 
With this observation, if $\boldsymbol{U}_t$ is fixed, we can bound this term following the arguments of the self-normalized bound for vector-valued martingales~\citep{abbasi2011improved}.

\begin{lemma}
\label{lemma: self-normalized bound}
For a fixed $\bar{\boldsymbol{U}}$, define $\bar{\boldsymbol{V}}_{t,i}(\lambda) \stackrel{\text { def }}{=} \left(\bar{\boldsymbol{U}}^{\top}\boldsymbol{X}_{t,i}  \right)\left(\bar{\boldsymbol{U}}^{\top}\boldsymbol{X}_{t,i} \right)^{\top}+\lambda \boldsymbol{I} $, then any $\delta > 0$, with probability at least $1-\delta$, for all $t \geq 0$,
\begin{align}
    & \sum_{i=1}^{M}\left\|\bar{\boldsymbol{U}}^{\top} \boldsymbol{X}_{t,i}\boldsymbol{\eta}_{t,i}\right\|^2_{\bar{\boldsymbol{V}}_{t,i}^{-1}} \\
    \leq &2  \log\left(\frac{\prod_{i=1}^{M}\left(\operatorname{det}(\bar{\boldsymbol{V}}_{t,i})^{1/2}\operatorname{det}(\lambda\boldsymbol{I})^{-1/2}\right)}{\delta}\right).
\end{align}
\end{lemma}

We defer the proof of Lemma~\ref{lemma: self-normalized bound} to Appendix~\ref{sec: proof of self-normalized bound}. We set $\lambda=1$. By Lemma~\ref{lemma: self-normalized bound}, we know that for a fixed $\bar{\boldsymbol{U}}$, with probability at least $1-\delta_1$,
\begin{align}
    \label{eqn: self-normalized bound for barU}
    \sum_{i=1}^{M}\left\|\sum_{n=1}^{t-1}\eta_{n,i}  \bar{\boldsymbol{U}}^{\top} x_{n,i}\right\|^2_{\bar{\boldsymbol{V}}^{-1}_{t,i}(\lambda)} \leq 2 \log\left(\frac{\prod_{i=1}^{M}\operatorname{det}(\bar{\boldsymbol{V}}_{t,i}(\lambda))^{1/2} \operatorname{det}(\lambda\boldsymbol{I})^{-1/2}}{\delta_1}\right) \leq 2Mk + 2 \log (1/\delta_1).
\end{align}

The above analysis shows that we can bound $\left\|\boldsymbol{\eta}_{t-1,i}^{\top} \boldsymbol{X}_{t-1,i}^{\top} \boldsymbol{U}_t\right\|_{\boldsymbol{V}^{-1}_{t-1,i}(\lambda)}$ if $\boldsymbol{U}_{t}$ is fixed as $\bar{\boldsymbol{U}}$. Following this idea, we  prove the lemma by the construction of $\epsilon$-net over all possible $\boldsymbol{U}_t$. To apply the trick of $\epsilon$-net, we need to slightly modify the derivation of Eqn~\ref{eqn: hTheta - Theta decomposition 0}. For a fixed matrix $\bar{\boldsymbol{U}} \in \mathbb{R}^{d \times 2k}$, we have
\begin{align}
    \label{eqn: hTheta - Theta decomposition 1-0}
    &\sum_{i=1}^{M}\left\|\hat{\boldsymbol{B}}_t \hat{\boldsymbol{w}}_{t,i}-\boldsymbol{B}\boldsymbol{w}_{i}\right\|^{2}_{\tilde{\boldsymbol{V}}_{t-1,i}(\lambda)} \\
    \label{eqn: hTheta - Theta decomposition 1-1}
     \leq & 2\sum_{i=1}^{M} \boldsymbol{\eta}_{t-1,i}^{\top} \boldsymbol{X}_{t-1,i}^{\top} \boldsymbol{U}_t \boldsymbol{r}_{t,i} + 4M \lambda \\
     \label{eqn: hTheta - Theta decomposition 1-2}
     = &  2\sum_{i=1}^{M} \boldsymbol{\eta}_{t-1,i}^{\top} \boldsymbol{X}_{t-1,i}^{\top} \bar{\boldsymbol{U}} \boldsymbol{r}_{t,i} + 2\sum_{i=1}^{M} \boldsymbol{\eta}_{t-1,i}^{\top} \boldsymbol{X}_{t-1,i}^{\top} \left(\boldsymbol{U}_t-\bar{\boldsymbol{U}} \right)\boldsymbol{r}_{t,i} + 4M \lambda \\ 
     \label{eqn: hTheta - Theta decomposition 1-3}
     \leq & 2\sum_{i=1}^{M} \left\|\boldsymbol{\eta}_{t-1,i}^{\top} \boldsymbol{X}_{t-1,i}^{\top} \bar{\boldsymbol{U}}\right\|_{\bar{\boldsymbol{V}}^{-1}_{t-1,i}(\lambda)} \left\|\boldsymbol{r}_{t,i}\right\|_{\bar{\boldsymbol{V}}_{t-1,i}(\lambda)} + 2\sum_{i=1}^{M} \boldsymbol{\eta}_{t-1,i}^{\top} \boldsymbol{X}_{t-1,i}^{\top} \left(\boldsymbol{U}_t-\bar{\boldsymbol{U}} \right)\boldsymbol{r}_{t,i} + 4M \lambda \\
     \label{eqn: hTheta - Theta decomposition 1-4}
     = & 2\sum_{i=1}^{M} \left\|\boldsymbol{\eta}_{t-1,i}^{\top} \boldsymbol{X}_{t-1,i}^{\top} \bar{\boldsymbol{U}}\right\|_{\bar{\boldsymbol{V}}^{-1}_{t-1,i}(\lambda)} \left\|\boldsymbol{r}_{t,i}\right\|_{\boldsymbol{V}_{t-1,i}(\lambda)} + 2\sum_{i=1}^{M} \boldsymbol{\eta}_{t-1,i}^{\top} \boldsymbol{X}_{t-1,i}^{\top} \left(\boldsymbol{U}_t-\bar{\boldsymbol{U}} \right)\boldsymbol{r}_{t,i}  \\
     & + 2\sum_{i=1}^{M} \left\|\boldsymbol{\eta}_{t-1,i}^{\top} \boldsymbol{X}_{t-1,i}^{\top} \bar{\boldsymbol{U}}\right\|_{\bar{\boldsymbol{V}}^{-1}_{t-1,i}(\lambda)} \left(\left\|\boldsymbol{r}_{t,i}\right\|_{\bar{\boldsymbol{V}}_{t-1,i}(\lambda)} - \left\|\boldsymbol{r}_{t,i}\right\|_{\boldsymbol{V}_{t-1,i}(\lambda)}\right) + 4M \lambda\\
     \label{eqn: hTheta - Theta decomposition 1-5}
     \leq & 2\sqrt{\sum_{i=1}^{M} \left\|\boldsymbol{\eta}_{t-1,i}^{\top} \boldsymbol{X}_{t-1,i}^{\top} \bar{\boldsymbol{U}}\right\|^2_{\boldsymbol{V}^{-1}_{t-1,i}(\lambda)}} \sqrt{\sum_{i=1}^{M}\left\|\hat{\boldsymbol{B}}_t \hat{\boldsymbol{w}}_{t,i}-\boldsymbol{B}\boldsymbol{w}_{i}\right\|^{2}_{\tilde{\boldsymbol{V}}_{t-1,i}(\lambda)}} + 2\sum_{i=1}^{M} \boldsymbol{\eta}_{t-1,i}^{\top} \boldsymbol{X}_{t-1,i}^{\top} \left(\boldsymbol{U}_t-\bar{\boldsymbol{U}} \right)\boldsymbol{r}_{t,i}  \\
     & + 2\sum_{i=1}^{M} \left\|\boldsymbol{\eta}_{t-1,i}^{\top} \boldsymbol{X}_{t-1,i}^{\top} \bar{\boldsymbol{U}}\right\|_{\bar{\boldsymbol{V}}^{-1}_{t-1,i}(\lambda)} \left(\left\|\boldsymbol{r}_{t,i}\right\|_{\bar{\boldsymbol{V}}_{t-1,i}(\lambda)} - \left\|\boldsymbol{r}_{t,i}\right\|_{\boldsymbol{V}_{t-1,i}(\lambda)}\right) + 4M \lambda
\end{align}

Eqn~\ref{eqn: hTheta - Theta decomposition 1-1}, \ref{eqn: hTheta - Theta decomposition 1-3} and \ref{eqn: hTheta - Theta decomposition 1-5} follow the same idea of Eqn~\ref{eqn: hTheta - Theta decomposition 3}, \ref{eqn: hTheta - Theta decomposition 4} and \ref{eqn: hTheta - Theta decomposition 6}. 

We construct an $\epsilon$-net $\mathcal{E}$ in Frobenius norm over the matrix set $\left\{\boldsymbol{U} \in \mathbb{R}^{d \times 2k} : \|\boldsymbol{U}\|_F \leq k \right\}$. It is not hard to see that $|\mathcal{E}| \leq \left(\frac{6\sqrt{2k}}{\epsilon}\right)^{2kd}$. By the union bound over all possible $\bar{\boldsymbol{U}} \in \mathcal{E}$, we know that with probability $1- |\mathcal{E}|\delta_1$, Eqn~\ref{eqn: self-normalized bound for barU} holds for any $\bar{\boldsymbol{U}} \in \mathcal{E}$. For each $\boldsymbol{U}_t$, we choose an $\bar{\boldsymbol{U}} \in \mathcal{E}$ with $\left\|\boldsymbol{U}_t - \bar{\boldsymbol{U}}\right\|_F \leq \epsilon$, and we have
\begin{align}
    \label{eqn: hTheta - Theta decomposition part 1}
     2\sqrt{\sum_{i=1}^{M} \left\|\boldsymbol{\eta}_{t-1,i}^{\top} \boldsymbol{X}_{t-1,i}^{\top} \bar{\boldsymbol{U}}\right\|^2_{\boldsymbol{V}^{-1}_{t-1,i}(\lambda)}} \leq 2 \sqrt{2Mk + 2 \log(1/\delta_1)}
\end{align}

Since $\left\|\boldsymbol{U}_t - \bar{\boldsymbol{U}}\right\|_F \leq \epsilon$, we have 
\begin{align}
    \label{eqn: hTheta - Theta decomposition part 2}
    2\sum_{i=1}^{M} \left\|\boldsymbol{\eta}_{t-1,i}^{\top} \boldsymbol{X}_{t-1,i}^{\top} \bar{\boldsymbol{U}}\right\|_{\bar{\boldsymbol{V}}^{-1}_{t-1,i}(\lambda)} \left(\left\|\boldsymbol{r}_{t,i}\right\|_{\bar{\boldsymbol{V}}_{t-1,i}(\lambda)} - \left\|\boldsymbol{r}_{t,i}\right\|_{\boldsymbol{V}_{t-1,i}(\lambda)}\right) \leq 2\sqrt{Mk\epsilon (2Mk+2 \log(1/\delta_1))}.
\end{align}

For the term $ 2\sum_{i=1}^{M} \boldsymbol{\eta}_{t-1,i}^{\top} \boldsymbol{X}_{t-1,i}^{\top} \left(\boldsymbol{U}_t-\bar{\boldsymbol{U}} \right)\boldsymbol{r}_{t,i}$, the following inequality holds for any step $t \in [T]$ with probability $1-MT\delta_2$, 
\begin{align}
    \label{eqn: hTheta - Theta decomposition part 3}
    2\sum_{i=1}^{M} \boldsymbol{\eta}_{t-1,i}^{\top} \boldsymbol{X}_{t-1,i}^{\top} \left(\boldsymbol{U}_t-\bar{\boldsymbol{U}} \right)\boldsymbol{r}_{t,i} 
    \leq &2 \sum_{i=1}^{M} \left\|\boldsymbol{\eta}_{t-1,i}\right\|_2 \left\|  \boldsymbol{X}_{t-1,i}^{\top} \left(\boldsymbol{U}_t-\bar{\boldsymbol{U}} \right)\boldsymbol{r}_{t,i}\right\|_2 \\
    \leq & 2 \sum_{i=1}^{M} \left\|\boldsymbol{\eta}_{t-1,i}\right\|_2 \sqrt{kT\epsilon} \\
    \leq & 2 M \sqrt{2\log(2/\delta_2) kT^2 \epsilon}
\end{align}

The last inequality follows from the fact that $|\eta_{n, i}| \leq \sqrt{2 \log(2 / \delta_2)}$ with probability $1 - \delta_2$ for fixed $n, i$, and apply a union bound over $n \in [t - 1], i \in [M]$. Plugging Eqn.~\ref{eqn: hTheta - Theta decomposition part 1}, \ref{eqn: hTheta - Theta decomposition part 2} and  \ref{eqn: hTheta - Theta decomposition part 3} back to Eqn.~\ref{eqn: hTheta - Theta decomposition 1-5}, the following inequality holds for any $t \in [T]$ with probability at least $1- |\mathcal{E}|\delta_1 - MT \delta_2$:
\begin{align}
    &\sum_{i=1}^{M}\left\|\hat{\boldsymbol{B}}_t \hat{\boldsymbol{w}}_{t,i}-\boldsymbol{B}\boldsymbol{w}_{i}\right\|^{2}_{\tilde{\boldsymbol{V}}_{t-1,i}(\lambda)} \\
    \leq & 2\sqrt{Mk + 2\log(1/\delta_1)} \sqrt{\sum_{i=1}^{M}\left\|\hat{\boldsymbol{B}}_t \hat{\boldsymbol{w}}_{t,i}-\boldsymbol{B}\boldsymbol{w}_{i}\right\|^{2}_{\tilde{\boldsymbol{V}}_{t-1,i}(\lambda)}} \\
    \label{formula:linear_rl_cross_reference}
    &+ 2 M \sqrt{2\log(2/\delta_2) kT^2 \epsilon} + 2\sqrt{Mk\epsilon (2Mk+2 \log(1/\delta_1))} + 4M\lambda
\end{align}

By solving the above inequality, we know that 
\begin{align}
    \sum_{i=1}^{M}\left\|\hat{\boldsymbol{B}}_t \hat{\boldsymbol{w}}_{t,i}-\boldsymbol{B}\boldsymbol{w}_{i}\right\|^{2}_{\tilde{\boldsymbol{V}}_{t-1,i}(\lambda)} \leq &32 \left(Mk +  \log(1/\delta_1)\right) + 4 M \sqrt{2\log(2/\delta_2) kT^2 \epsilon} \\
    &+ 4\sqrt{Mk\epsilon (2Mk+2 \log(1/\delta_1))} + 8M\lambda
\end{align}

Setting $\lambda=1$, $\epsilon = \frac{1}{kM^2T^2}$, $\delta_1 = \frac{\delta}{2 \left(\frac{6\sqrt{2k}}{\epsilon}\right)^{2kd}} \leq \frac{\delta}{2 |\mathcal{E}|}$, and $\delta_2 = \frac{\delta}{2MT}$, the following inequality holds with probability $1-\delta$:
\begin{align}
    \sum_{i=1}^{M}\left\|\hat{\boldsymbol{B}}_t \hat{\boldsymbol{w}}_{t,i}-\boldsymbol{B}\boldsymbol{w}_{i}\right\|^{2}_{\tilde{\boldsymbol{V}}_{t-1,i}(\lambda)} & \leq L \defeq 48\left(Mk + 5kd \log(kMT)\right) + 32 \log(4MT) + 76 \log(1 / \delta)
\end{align}

At last we talk about the trivial setting where $k < d < 2k$. In this case, we can write $\hat{\boldsymbol{\Theta}}_t - \boldsymbol{\Theta} = \boldsymbol{R}_t$ where $\boldsymbol{R}_t \in \mathbb{R}^{d \times M}$. The proof then follows the same framework as the case when $d \geq 2k$, except that we don't need to consider $\boldsymbol{U}_t$ and construct $\epsilon$-net over all possible $\boldsymbol{U}_t$. It is not hard to show that $\sum_{i=1}^{M}\left\|\hat{\boldsymbol{B}}_t \hat{\boldsymbol{w}}_{t,i}-\boldsymbol{B}\boldsymbol{w}_{i}\right\|^{2}_{\tilde{\boldsymbol{V}}_{t-1,i}(\lambda)} \leq 24 \left(Md + 2 \log(Tk/\delta)\right)$ in this case, which is also less than $L$ since $d < 2k$. 
\end{proof}

\newpage 

\subsubsection{Proof of Theorem~\ref{thm: main theory for linear bandits}}
\label{sec: proof of main theorem for linear bandits}
With Lemma~\ref{lemma: confidence set for linear bandits}, we are ready to prove Theorem~\ref{thm: main theory for linear bandits}.

\begin{proof}
Let $\tilde{\boldsymbol{V}}_{t, i}(\lambda) = \boldsymbol{X}_{t,i}\boldsymbol{X}_{t,i}^{\top} + \lambda \boldsymbol{I}_d$ for some $\lambda > 0$.
\begin{align}
\mathrm{Reg}(T) & = \sum_{t = 1}^T \sum_{i = 1}^M \left\langle \boldsymbol{\theta}_i, \boldsymbol{x}^*_{t, i} - \boldsymbol{x}_{t, i} \right\rangle \\
& \leq \sum_{t = 1}^T \sum_{i = 1}^M \left\langle \tilde{\boldsymbol{\theta}}_{t, i} - \boldsymbol{\theta}_i, \boldsymbol{x}_{t, i} \right\rangle \\
& = \sum_{t = 1}^T \sum_{i = 1}^M \left\langle \tilde{\boldsymbol{\theta}}_{t, i} - \hat{\boldsymbol{\theta}}_{t , i} + \hat{\boldsymbol{\theta}}_{t , i} - \boldsymbol{\theta}_i, \boldsymbol{x}_{t, i} \right\rangle\\
& \leq \sum_{t = 1}^T \sum_{i = 1}^M \left(\left\| \tilde{\boldsymbol{\theta}}_{t, i} - \hat{\boldsymbol{\theta}}_{t , i}\right\|_{\tilde{\boldsymbol{V}}_{t - 1, i}(\lambda)} + \left\|\hat{\boldsymbol{\theta}}_{t , i} - \boldsymbol{\theta}_i\right\|_{\tilde{\boldsymbol{V}}_{t - 1, i}(\lambda)}\right) \left\|\boldsymbol{x}_{t, i} \right\|_{\tilde{\boldsymbol{V}}_{t - 1, i}(\lambda)^{-1}} \\
& \leq  \left(\sqrt{\sum_{t = 1}^T\sum_{i = 1}^M \left\| \tilde{\boldsymbol{\theta}}_{t, i} - \hat{\boldsymbol{\theta}}_{t , i}\right\|_{\tilde{\boldsymbol{V}}_{t - 1, i}(\lambda)}^2} + \sqrt{\sum_{i = 1}^M \left\|\hat{\boldsymbol{\theta}}_{t , i} - \boldsymbol{\theta}_i\right\|_{\tilde{\boldsymbol{V}}_{t - 1, i}(\lambda)}^2}\right) \cdot \sqrt{\sum_{t = 1}^T\sum_{i = 1}^M \left\|\boldsymbol{x}_{t, i} \right\|_{\tilde{\boldsymbol{V}}_{t - 1, i}(\lambda)^{-1}}^2}\\
& \leq 2\sqrt{T\left(L + 4\lambda M\right)} \cdot \sqrt{\sum_{i = 1}^M \sum_{t = 1}^T \left\|\boldsymbol{x}_{t, i} \right\|_{\tilde{\boldsymbol{V}}_{t - 1, i}(\lambda)^{-1}}^2}
\end{align}

where the first inequality is due to $\sum_{i=1}^{M} \left\langle \boldsymbol{\theta}_i, \boldsymbol{x}^*_{t, i} \right\rangle \leq \left\langle \tilde{\boldsymbol{\theta}}_{t,i}, \boldsymbol{x}_{t, i} \right\rangle$ from the optimistic choice of $\tilde{\boldsymbol{\theta}}_{t,i}$ and $\boldsymbol{x}_{t, i}$. By Lemma 11 of \citet{abbasi2011improved}, as long as $\lambda \geq 1$ we have
\begin{align}
    \label{eqn: linear bandit potential function}
    \sum_{t = 1}^T \left\|\boldsymbol{x}_{t, i} \right\|_{\boldsymbol{\tilde{V}}_{t - 1, i}(\lambda')^{-1}}^2 \leq 2 \log \frac{\det(\boldsymbol{\tilde{V}}_{T, i}(\lambda'))}{\det(\lambda' \boldsymbol{I}_d)} \leq 2 d\log\left(1 + \frac{T}{\lambda d}\right)
\end{align}
$$  $$

Therefore, we can finally bound the regret by choosing $\lambda = 1$
\begin{align}
\mathrm{Reg}(T) & \leq 2\sqrt{T(L + 4M)} \cdot \sqrt{ \sum_{i = 1}^M \sum_{t = 1}^T \left\|\boldsymbol{x}_{t, i} \right\|_{\tilde{\boldsymbol{V}}_{t - 1, i}(\lambda')^{-1}}^2} \\
& \leq 2 \sqrt{T  \left(L + 4M\right)} \cdot \sqrt{Md \log \left(1 + \frac{T}{d}\right)} \\
& = \tilde{O}\left(M\sqrt{dkT} + d\sqrt{kMT} \right).
\end{align}
\end{proof}

\newpage

\subsubsection{Proof of Lemma~\ref{lemma: self-normalized bound}}
\label{sec: proof of self-normalized bound}

The proof of Lemma~\ref{lemma: self-normalized bound} follows the similar idea of Theorem~1 in \citet{abbasi2011improved}. We consider the $\sigma$-algebra $F_t = \sigma\left(\{\boldsymbol{x}_{1,i}\}_{i=1}^{M},\{\boldsymbol{x}_{2,i}\}_{i=1}^{M},\cdots, \{\boldsymbol{x}_{t+1,i}\}_{i=1}^{M}, \{\eta_{1,i}\}_{i=1}^{M}, \{\eta_{2,i}\}_{i=1}^{M}, \cdots, \{\eta_{t,i}\}_{i=1}^{M}\right)$, then $\{\boldsymbol{x}_{t,i}\}_{i=1}^{M}$ is $F_{t-1}$-measurable, and $\{\eta_{t,i}\}_{i=1}^{M}$ is $F_t$-measurable.

Define $\bar{\boldsymbol{x}}_{t,i} = \boldsymbol{U}^{\top} \boldsymbol{x}_{t,i}$ and $\boldsymbol{S}_{t,i} = \sum_{n=1}^{t} \bar{\boldsymbol{U}}^{\top}\boldsymbol{x}_{t,i}\eta_{t,i}$. Let
\begin{align}
    M_t(\boldsymbol{Q}) = \exp\left(\sum_{n=1}^{t} \sum_{i=1}^{M}\left[\eta_{t,i} \left \langle \boldsymbol{q}_i,\bar{\boldsymbol{x}}_{t,i} \right \rangle - \frac{1}{2} \left \langle \boldsymbol{q}_i,\bar{\boldsymbol{x}}_{t,i} \right \rangle^2\right]\right), \quad \boldsymbol{Q} = [\boldsymbol{q}_1,\cdots, \boldsymbol{q}_M] \in \mathbb{R}^{2k\times M}
\end{align}

\begin{lemma}
\label{lemma: supermatingale}
Let $\tau$ be a stopping time w.r.t the filtration $\{F_t\}_{t=0}^{\infty}$. Then $M_t(\boldsymbol{Q})$ is almost surely well-defined and $\mathbb{E}[M_t(\boldsymbol{Q})] \leq 1$.
\end{lemma}
\begin{proof}
Let $D_t(\boldsymbol{Q}) = \exp \left(\sum_{i=1}^{M}\left[\eta_{t,i} \left \langle \boldsymbol{q}_i,\bar{\boldsymbol{x}}_{t,i} \right \rangle - \frac{1}{2} \left \langle \boldsymbol{q}_i,\bar{\boldsymbol{x}}_{t,i} \right \rangle^2\right]\right)$. By the sub-Gaussianity of $\eta_{t,i}$, we have 
\begin{align}
    \mathbb{E}\left[\exp\left(\left[\eta_{t,i} \left \langle \boldsymbol{q}_i,\bar{\boldsymbol{x}}_{t,i} \right \rangle - \frac{1}{2} \left \langle \boldsymbol{q}_i,\bar{\boldsymbol{x}}_{t,i} \right \rangle^2\right]\right)] \mid F_{t-1} \right] \leq 1.
\end{align}

Then we have $\mathbb{E}\left[D_t(\boldsymbol{Q}) \mid F_{t-1}\right] \leq 1$. Further,
\begin{align}
    \mathbb{E}\left[M_t(\boldsymbol{Q}) \mid F_{t-1}\right] &= \mathbb{E}\left[M_1(\boldsymbol{Q}) \cdots D_{t-1}(\boldsymbol{Q}) D_{t}(\boldsymbol{Q})\mid F_{t-1}\right] \\
    &= D_1(\boldsymbol{Q}) \cdots D_{t-1}(\boldsymbol{Q})\mathbb{E}\left[ D_{t}(\boldsymbol{Q})\mid F_{t-1}\right] \leq M_{t-1}(\boldsymbol{Q})
\end{align}
This shows that $\{M_{t}(\boldsymbol{Q})\}_{t=0}^{\infty}$ is a supermartingale and $\mathbb{E}\left[M_t(\boldsymbol{Q})\right] \leq 1$.

Following the same argument of Lemma~8 in \citet{abbasi2011improved}, we show that $M_{\tau}(\boldsymbol{Q})$ is almost surely well-defined. By the convergence theorem for nonnegative supermartingales, $M_{\infty}(\boldsymbol{Q}) = \lim_{t \rightarrow \infty} M_t(\boldsymbol{Q})$ is almost surely well-defined. Therefore, $M_{\tau}(\boldsymbol{Q})$ is indeed well-defined independently of whether $\tau < \infty$ or not. Let $W_t(\boldsymbol{Q}) = M_{\min\{\tau, t\}}(\boldsymbol{Q})$ be a stopped version of $(M_t(\boldsymbol(Q)))_t$. By Fatou's Lemma, $\mathbb{E}[M_{\tau}(\boldsymbol{Q})] = \mathbf{E}\left[\liminf _{t \rightarrow \infty} W_{t}(\boldsymbol{Q})\right] \leq \liminf _{t \rightarrow \infty} \mathbf{E}\left[W_{t}(\boldsymbol{Q})\right] \leq 1$. This shows that $\mathbb{E}[M_{\tau}(\boldsymbol{Q})] \leq 1$. 
\end{proof}

The next lemma uses the ``method of mixtures'' technique to bound $\sum_{i=1}^{M} \|\boldsymbol{S}_{t,i}\|^2_{\bar{\boldsymbol{V}}_{t,i}^{-1}(\lambda)}$.

\begin{lemma}
\label{lemma: self-normalized bound with stopping time}
Let $\tau$ be a stopping time w.r.t the filtration $\{F_t\}_{t=0}^{\infty}$. Then, for $\delta > 0$, with probability $1-\delta$,
\begin{align}
    \sum_{i=1}^{M} \|\boldsymbol{S}_{\tau,i}\|^2_{\bar{\boldsymbol{V}}_{\tau,i}^{-1}(\lambda)} \leq 2 \log\left(\frac{\prod_{i=1}^{M}\left(\operatorname{det}(\bar{\boldsymbol{V}}_{\tau,i})^{1/2} \operatorname{det}(\lambda \boldsymbol{I})^{-1/2}\right)}{\delta}\right).
\end{align}

\end{lemma}
\begin{proof}
For each $i \in [M]$, let $\boldsymbol{\Lambda}_i$ be a $\mathbb{R}^{2k}$ Gaussian random variable which is independent of all the other random variables and whose covariance is $\lambda^{-1}\boldsymbol{I}$. Define $M_t = \mathbb{E}\left[M_t([\boldsymbol{\Lambda}_1,\cdots, \boldsymbol{\Lambda}_M])\mid F_{\infty}\right]$. We still have $\mathbb{E}[M_{\tau}] = \mathbb{E}[\mathbb{E}[M_t([\boldsymbol{\Lambda}_1,\cdots, \boldsymbol{\Lambda}_M]) \mid \{\boldsymbol{\Lambda}_i\}_{i=1}^{M}]] \leq 1$.

Now we calculate $M_t$. Define $M_{t,i}(\boldsymbol{q}_i) \stackrel{\text { def }}{=} \exp\left(\sum_{n=1}^{t} \left[\eta_{t,i} \left \langle \boldsymbol{q}_i,\bar{\boldsymbol{x}}_{t,i} \right \rangle - \frac{1}{2} \left \langle \boldsymbol{q}_i,\bar{\boldsymbol{x}}_{t,i} \right \rangle^2\right]\right)$, then we have $M_t = \mathbb{E}\left[\prod_{i=1}^{M} M_{t,i}(\boldsymbol{\Lambda_i})\mid F_{\infty}\right] = \prod_{i=1}^{M} \mathbb{E}\left[ M_{t,i}(\boldsymbol{\Lambda_i})\mid F_{\infty}\right]$, where the second equality is due to the fact that $\{M_{t,i}(\boldsymbol{\Lambda}_i)\}_{i=1}^{M}$ are relatively independent given $F_{\infty}$. We only need to calculate $\mathbb{E}\left[ M_{t,i}(\boldsymbol{\Lambda}_i)\mid F_{\infty}\right]$ for each $i \in [M]$.

Following the proof of Lemma~9 in \citet{abbasi2011improved}, we know that 
\begin{align}
    \mathbb{E}\left[ M_{t,i}(\boldsymbol{\Lambda}_i)\mid F_{\infty}\right] = \left(\frac{\operatorname{det}(\lambda \boldsymbol{I})}{\operatorname{det}(\bar{\boldsymbol{V}}_{t,i})}\right)^{1/2} \exp \left(\frac{1}{2} \|\boldsymbol{S}_{t,i}\|^2_{\bar{\boldsymbol{V}}_{t,i}^{-1}(\lambda)}\right).
\end{align}

Then we have 
\begin{align}
    M_t = \prod_{i=1}^{M}\left(\left(\frac{\operatorname{det}(\lambda \boldsymbol{I})}{\operatorname{det}(\bar{\boldsymbol{V}}_{t,i})}\right)^{1/2} \right) \exp \left(\frac{1}{2} \sum_{i=1}^{M}\|\boldsymbol{S}_{t,i}\|^2_{\bar{\boldsymbol{V}}_{t,i}^{-1}(\lambda)}\right).
\end{align}

Since $\mathbb{E}[M_{\tau}] \leq 1$, we have
\begin{align*}
    \lefteqn{\operatorname{Pr}\left[ \sum_{i=1}^{M} \|\boldsymbol{S}_{\tau,i}\|^2_{\bar{\boldsymbol{V}}_{\tau,i}^{-1}(\lambda)} > 2 \log\left(\frac{\prod_{i=1}^{M}\left(\operatorname{det}(\bar{\boldsymbol{V}}_{\tau,i})^{1/2} \operatorname{det}(\lambda \boldsymbol{I})^{-1/2}\right)}{\delta}\right)\right]} \\
    = & \operatorname{Pr}\left[\frac{\exp\left(\frac{1}{2}\sum_{i=1}^{M} \|\boldsymbol{S}_{\tau,i}\|^2_{\bar{\boldsymbol{V}}_{\tau,i}^{-1}(\lambda)}\right)}{\delta^{-1} \left(\prod_{i=1}^{M}\left(\operatorname{det}(\bar{\boldsymbol{V}}_{t,i})^{1/2} \operatorname{det}(\lambda \boldsymbol{I})^{-1/2}\right)\right)} > 1\right] \\
    \leq & \mathbb{E}\left[\frac{\exp\left(\sum_{i=1}^{M} \|\boldsymbol{S}_{\tau,i}\|^2_{\bar{\boldsymbol{V}}_{\tau,i}^{-1}(\lambda)}\right)}{\delta^{-1} \left(\prod_{i=1}^{M}\left(\operatorname{det}(\bar{\boldsymbol{V}}_{\tau,i})^{1/2} \operatorname{det}(\lambda \boldsymbol{I})^{-1/2}\right)\right)} \right] \\
    = &\mathbb{E}[M_{\tau}] \delta \leq \delta.
\end{align*}
\end{proof}

\begin{proof}
(Proof of Lemma~\ref{lemma: self-normalized bound}) The only remaining issue is the stopping time construction. Define the bad event
\begin{align}
    B_t(\delta) \stackrel{\text{def}}{=} \left\{ \omega \in \Omega:  \sum_{i=1}^{M} \|\boldsymbol{S}_{t,i}\|^2_{\bar{\boldsymbol{V}}_{t,i}^{-1}(\lambda)} > 2 \log\left(\frac{\prod_{i=1}^{M}\left(\operatorname{det}(\bar{\boldsymbol{V}}_{t,i})^{1/2} \operatorname{det}(\lambda \boldsymbol{I})^{-1/2}\right)}{\delta}\right)\right\}
\end{align}

Consider the stopping time $\tau(\omega) = \min\{t \geq 0: \omega \in B_t(\delta)\}$, we have $\bigcup_{t \geq 0} B_{t}(\delta)=\{\omega: \tau(\omega)<\infty\}$.

By lemma~\ref{lemma: self-normalized bound with stopping time}, we have 
\begin{align}
    \operatorname{Pr}\left[\bigcup_{t \geq 0} B_{t}(\delta)\right]  =&\operatorname{Pr}[\tau<\infty] \\
     = & \operatorname{Pr} \left[ \sum_{i=1}^{M} \|\boldsymbol{S}_{\tau,i}\|^2_{\bar{\boldsymbol{V}}_{\tau,i}^{-1}(\lambda)} > 2 \log\left(\frac{\prod_{i=1}^{M}\left(\operatorname{det}(\bar{\boldsymbol{V}}_{\tau,i})^{1/2} \operatorname{det}(\lambda \boldsymbol{I})^{-1/2}\right)}{\delta}\right), \tau \leq \infty\right] \\
     \leq &  \operatorname{Pr} \left[ \sum_{i=1}^{M} \|\boldsymbol{S}_{\tau,i}\|^2_{\bar{\boldsymbol{V}}_{\tau,i}^{-1}(\lambda)} > 2 \log\left(\frac{\prod_{i=1}^{M}\left(\operatorname{det}(\bar{\boldsymbol{V}}_{\tau,i})^{1/2} \operatorname{det}(\lambda \boldsymbol{I})^{-1/2}\right)}{\delta}\right)\right] \\
     \leq &\delta.
\end{align}
\end{proof}

\newpage

\subsubsection{Proof of Theorem~\ref{thm: regret for misspecified linear bandits}}
\label{sec: proof for misspecified linear bandits}

\begin{proof}
 The proof follows the same idea of that for Theorem~\ref{thm: main theory for linear bandits}. The only difference is that, in our setting, we have $y_{t,i} = \boldsymbol{x}_{t,i}^{\top}\boldsymbol{B}\boldsymbol{w}_i + \eta_{t,i} + \Delta_{t,i}$, where  $\boldsymbol{\theta}_{i} = \boldsymbol{B} \boldsymbol{w}_i $ is the best approximator for task $i \in [M]$ such that $\left|\mathbb{E}\left[y_{i} \mid \boldsymbol{x}_{i}\right]-\left\langle\boldsymbol{x}_{i}, \dot{\boldsymbol{B}} \dot{\boldsymbol{w}}_{i}\right\rangle\right| \leq \zeta$, and $\|\Delta_{t,i}\| \leq \zeta$. Define $\boldsymbol{\Delta}_{t, i}=\left[\Delta_{1, i}, \Delta_{2, i}, \cdots, \Delta_{t, i}\right]$. Similarly, by the optimality of $\hat{\boldsymbol{B}}_t$ and $\hat{\boldsymbol{W}}_t = [\hat{\boldsymbol{w}}_{t,1}, \cdots, \hat{\boldsymbol{w}}_{t,M}]$, we know that $\sum_{i=1}^{M} \left\|\boldsymbol{y}_{t-1,i}-\boldsymbol{X}_{t-1,i}^{\top} \hat{\boldsymbol{B}}_t \hat{\boldsymbol{w}}_{t,i}\right\|^{2}_2 \leq \sum_{i=1}^{M} \left\|\boldsymbol{y}_{t-1,i}-\boldsymbol{X}_{t-1,i}^{\top} \boldsymbol{B} \boldsymbol{w}_{i}\right\|^{2}$. Since $\boldsymbol{y}_{t-1, i}=\boldsymbol{X}_{t-1, i}^{\top} \boldsymbol{B} \boldsymbol{w}_{i}+\boldsymbol{\eta}_{t-1, i}+\boldsymbol{\Delta}_{t, i}$,  thus we have
\begin{align}
    \label{eqn: optimality of hTheta misspecified setting}
    & \sum_{i=1}^{M} \left\|\boldsymbol{X}_{t-1,i}^{\top} \left(\hat{\boldsymbol{B}}_t \hat{\boldsymbol{w}}_{t,i} - \boldsymbol{B} \boldsymbol{w}_{i}\right)\right\|^2 \\
    \leq &2\sum_{i=1}^{M} \boldsymbol{\eta}_{t-1,i}^{\top} \boldsymbol{X}_{t-1,i}^{\top} \left(\hat{\boldsymbol{B}}_t \hat{\boldsymbol{w}}_{t,i} - \boldsymbol{B} \boldsymbol{w}_{i}\right)  + 2 \sum_{i=1}^{M} \boldsymbol{\Delta}_{t-1, i}^{\top} \boldsymbol{X}_{t-1, i}^{\top}\left(\hat{\boldsymbol{B}}_{t} \hat{\boldsymbol{w}}_{t, i}-\boldsymbol{B} \boldsymbol{w}_{i}\right) \\
    \leq & 2\sum_{i=1}^{M} \boldsymbol{\eta}_{t-1,i}^{\top} \boldsymbol{X}_{t-1,i}^{\top} \left(\hat{\boldsymbol{B}}_t \hat{\boldsymbol{w}}_{t,i} - \boldsymbol{B} \boldsymbol{w}_{i}\right) + 2 \sum_{i=1}^{M} \left\|\boldsymbol{X}_{t-1,i}\boldsymbol{\Delta}_{t-1,i}\right\|_{\tilde{\boldsymbol{V}}_{t-1,i}^{-1}(\lambda)} \left\|\hat{\boldsymbol{B}}_t \hat{\boldsymbol{w}}_{t,i} - \boldsymbol{B} \boldsymbol{w}_{i}\right\|_{\tilde{\boldsymbol{V}}_{t-1,i}(\lambda)} \\
    \leq & 2\sum_{i=1}^{M} \boldsymbol{\eta}_{t-1,i}^{\top} \boldsymbol{X}_{t-1,i}^{\top} \left(\hat{\boldsymbol{B}}_t \hat{\boldsymbol{w}}_{t,i} - \boldsymbol{B} \boldsymbol{w}_{i}\right) + 2 \sum_{i=1}^{M} \sqrt{T}\zeta \left\|\hat{\boldsymbol{B}}_t \hat{\boldsymbol{w}}_{t,i} - \boldsymbol{B} \boldsymbol{w}_{i}\right\|_{\tilde{\boldsymbol{V}}_{t-1,i}(\lambda)} \\
    \label{eqn: misspecified bandit, confidence proof}
    \leq & 2\sum_{i=1}^{M} \boldsymbol{\eta}_{t-1,i}^{\top} \boldsymbol{X}_{t-1,i}^{\top} \left(\hat{\boldsymbol{B}}_t \hat{\boldsymbol{w}}_{t,i} - \boldsymbol{B} \boldsymbol{w}_{i}\right) + 2 \sqrt{MT} \zeta \sqrt{\sum_{i=1}^{M}\left\|\hat{\boldsymbol{B}}_t \hat{\boldsymbol{w}}_{t,i} - \boldsymbol{B} \boldsymbol{w}_{i}\right\|^2_{\tilde{\boldsymbol{V}}_{t-1,i}(\lambda)}}
\end{align}
The third inequality follows from Projection Bound (Lemma 8) in \citet{zanette2020learning}. The first term of Eqn~\ref{eqn: misspecified bandit, confidence proof} shares the same form of Eqn~\ref{eqn: optimality of hTheta}. Following the same proof idea of Lemma~\ref{lemma: confidence set for linear bandits}, we know that with probability $1-\delta$,
\begin{align}
    &\sum_{i=1}^{M}\left\|\hat{\boldsymbol{B}}_t \hat{\boldsymbol{w}}_{t,i}-\boldsymbol{B}\boldsymbol{w}_{i}\right\|^{2}_{\tilde{\boldsymbol{V}}_{t-1,i}(\lambda)} \\
    \leq  &\left(2\sqrt{Mk + 8kd\log(kMT/\delta)}+ 2\sqrt{MT}\zeta\right) \sqrt{\sum_{i=1}^{M}\left\|\hat{\boldsymbol{B}}_t \hat{\boldsymbol{w}}_{t,i}-\boldsymbol{B}\boldsymbol{w}_{i}\right\|^{2}_{\tilde{\boldsymbol{V}}_{t-1,i}(\lambda)}} + 4M + 4 \sqrt{\log(4MT/\delta)}
\end{align}

Solving for $\sum_{i=1}^{M}\left\|\hat{\boldsymbol{B}}_t \hat{\boldsymbol{w}}_{t,i}-\boldsymbol{B}\boldsymbol{w}_{i}\right\|^{2}_{\tilde{\boldsymbol{V}}_{t-1,i}(\lambda)}$, we know that the true parameter $\boldsymbol{B}\boldsymbol{W}$ is always contained in the confidence set, i.e.
\begin{align}
\label{eqn: confidence bound for misspecified linear bandit}
\sum_{i=1}^{M}\left\|\hat{\boldsymbol{B}}_{t} \hat{\boldsymbol{w}}_{t, i}-\boldsymbol{B} \boldsymbol{w}_{i}\right\|_{\tilde{\boldsymbol{V}}_{t-1, i}(\lambda)}^{2} \leq L',
\end{align}
where $L' = 2L + 32 MT \zeta^2$.

Thus we have 
\begin{align}
\mathrm{Reg}(T) & = \sum_{t = 1}^T \sum_{i = 1}^M \left(y^*_{t,i}- y_{t,i}\right) \\
& \leq 2MT\zeta + \sum_{t = 1}^T \sum_{i = 1}^M \left\langle \boldsymbol{\theta}_i, \boldsymbol{x}^*_{t, i} - \boldsymbol{x}_{t, i} \right\rangle \\
& \leq 2MT\zeta +\sum_{t = 1}^T \sum_{i = 1}^M \left\langle \tilde{\boldsymbol{\theta}}_{t, i} - \boldsymbol{\theta}_i, \boldsymbol{x}_{t, i} \right\rangle \\
& = 2MT\zeta +\sum_{t = 1}^T \sum_{i = 1}^M \left\langle \tilde{\boldsymbol{\theta}}_{t, i} - \hat{\boldsymbol{\theta}}_{t , i} + \hat{\boldsymbol{\theta}}_{t , i} - \boldsymbol{\theta}_i, \boldsymbol{x}_{t, i} \right\rangle\\
& \leq 2MT\zeta +\sum_{t = 1}^T \sum_{i = 1}^M \left(\left\| \tilde{\boldsymbol{\theta}}_{t, i} - \hat{\boldsymbol{\theta}}_{t , i}\right\|_{\tilde{\boldsymbol{V}}_{t - 1, i}(\lambda)} + \left\|\hat{\boldsymbol{\theta}}_{t , i} - \boldsymbol{\theta}_i\right\|_{\tilde{\boldsymbol{V}}_{t - 1, i}(\lambda)}\right) \left\|\boldsymbol{x}_{t, i} \right\|_{\tilde{\boldsymbol{V}}_{t - 1, i}(\lambda)^{-1}} \\
& \leq  2MT\zeta +\left(\sqrt{\sum_{t = 1}^T\sum_{i = 1}^M \left\| \tilde{\boldsymbol{\theta}}_{t, i} - \hat{\boldsymbol{\theta}}_{t , i}\right\|_{\tilde{\boldsymbol{V}}_{t - 1, i}(\lambda)}^2} + \sqrt{\sum_{i = 1}^M \left\|\hat{\boldsymbol{\theta}}_{t , i} - \boldsymbol{\theta}_i\right\|_{\tilde{\boldsymbol{V}}_{t - 1, i}(\lambda)}^2}\right) \cdot \sqrt{\sum_{t = 1}^T\sum_{i = 1}^M \left\|\boldsymbol{x}_{t, i} \right\|_{\tilde{\boldsymbol{V}}_{t - 1, i}(\lambda)^{-1}}^2}\\
& \leq 2MT\zeta +2\sqrt{T\left(L' + 4\lambda M\right)} \cdot \sqrt{\sum_{i = 1}^M \sum_{t = 1}^T \left\|\boldsymbol{x}_{t, i} \right\|_{\tilde{\boldsymbol{V}}_{t - 1, i}(\lambda)^{-1}}^2} \\
& \leq 2MT\zeta +2\sqrt{T\left(L' + 4\lambda M\right)} \sqrt{Md \log(1+\frac{T}{d})} \\
& = \tilde{O}(M\sqrt{dkT} + d\sqrt{kMT} + MT\sqrt{d}\zeta),
\end{align}
where the second inequality is due to $\sum_{i=1}^{M} \left\langle \boldsymbol{\theta}_i, \boldsymbol{x}^*_{t, i} \right\rangle \leq \left\langle \tilde{\boldsymbol{\theta}}_{t,i}, \boldsymbol{x}_{t, i} \right\rangle$ from the optimistic choice of $\tilde{\boldsymbol{\theta}}_{t,i}$ and $\boldsymbol{x}_{t, i}$. The third inequality is due to Eqn~\ref{eqn: confidence bound for misspecified linear bandit}. The last inequality is from Eqn~\ref{eqn: linear bandit potential function}.
\end{proof}

\newpage

\subsubsection{Proof of Theorem~\ref{thm: lower bound for linear bandits}}
\label{sec: proof of lower bound for linear bandits}

 Since our setting is strictly harder than the setting of multi-task linear bandit with infinite arms in \citet{yang2020provable}, we can prove the following lemma directly from their Theorem 4 by reduction.
 \begin{lemma}
 \label{lemma: lower bound for linear bandit original terms}
Under the setting of Theorem~\ref{thm: lower bound for linear bandits}, the regret of any Algorithm~$\mathcal{A}$ is lower bounded by $ \Omega\left(Mk\sqrt{T}+ d\sqrt{kMT}\right).$
\end{lemma}

In order to prove Theorem~\ref{thm: lower bound for linear bandits}, we only need to show that the following lemma is true.

\begin{lemma}
\label{lemma: lower bound for linear bandits, approximation error term}
Under the setting of Theorem~\ref{thm: lower bound for linear bandits}, the regret of any Algorithm~$\mathcal{A}$ is lower bounded by $ \Omega\left(MT\sqrt{d}\zeta\right).$
\end{lemma}

\begin{proof}
 (Proof of Lemma~\ref{lemma: lower bound for linear bandits, approximation error term}) 
 
 To prove Lemma~\ref{lemma: lower bound for linear bandits, approximation error term}, we leverage the lower bound for misspecified linear bandits in the single-task setting. We restate the following lemma from the previous literature with a slight modification of notations. 

\begin{lemma}
\label{lemma: related lemma for linear bandit lower bound}
(Proposition 6 in \citet{zanette2020learning}). There exists a feature map $\phi:\mathcal{A} \rightarrow \mathbb{R}^d$ that defines a misspecified linear bandits class $\mathcal{M}$ such that every bandit instance in that class has reward response:
\begin{align*}
    \mu_a = \boldsymbol{\phi}_a^{\top} \boldsymbol{\theta} + z_a
\end{align*}
for any action $a$ (Here $z_a \in [0,\zeta]$ is the deviation from linearity and $\mu_a \in [0,1]$) and such that the expected regret of any algorithm on at least a member of the class up to round $T$ is $\Omega(\sqrt{d}\zeta T)$.
\end{lemma}

Suppose $M$ can be exactly divided by $k$, we construct the following instances to prove lemma~\ref{lemma: lower bound for linear bandits, approximation error term}. We divide $M$ tasks into $k$ groups. Each group shares the same parameter $\theta_i$. To be more specific, we let $\boldsymbol{w}_1 = \boldsymbol{w}_2 = \cdots = \boldsymbol{w}_{M/k} = \boldsymbol{e}_1$, $\boldsymbol{w}_{M/k+1} = \boldsymbol{w}_{M/k+2} = \cdots = \boldsymbol{w}_{2M/k} = \boldsymbol{e}_2$, $\cdots$, $\boldsymbol{w}_{(k-1)M/k+1} = \boldsymbol{w}_{(k-1)M/k+2} = \cdots = \boldsymbol{w}_{M} = \boldsymbol{e}_k$. Under this construction, the parameters $\theta_i$ for these tasks are exactly the same in each group, but relatively independent among different groups. That is to say, the expected regret lower bound is at least the summation of the regret lower bounds in all $k$ groups.

Now we consider the regret lower bound for group $j \in [k]$. Since the parameters are shared in the same group, the regret of running an algorithm for $M/k$ tasks with $T$ steps each is at least the regret of running an algorithm for single-task linear bandit with $M/k \cdot T$ steps. By Lemma~\ref{lemma: related lemma for linear bandit lower bound}, the regret for single-task linear bandit with $MT/k$ steps is at least $\Omega(\sqrt{d} \zeta MT/k)$. Summing over all $k$ groups, we can prove that the regret lower bound is $\Omega(\sqrt{d} \zeta MT)$.
\end{proof}

Combining Lemma~\ref{lemma: lower bound for linear bandit original terms} and Lemma~\ref{lemma: lower bound for linear bandits, approximation error term}, we complete the proof of Theorem~\ref{thm: lower bound for linear bandits}.

\newpage

\subsection{Proof of Theorem~\ref{theorem:linear_rl_regret_bound}}
\label{sec: omiited proof in linea_rl}

\subsubsection{Definitions and First Step Analysis}
\label{appendix_linear_rl:definitions}

Before presenting the proof of theorem \ref{theorem:linear_rl_regret_bound}, we will make a first step analysis on the low-rank least-square estimator in equation \ref{formula:linear_rl_least_square}.

For any $\left\{Q_{h+1}^i\right\}_{i=1}^M \in \caQ_{h+1}$, there exists $\left\{\dot{\boldsymbol{\theta}}_{h}^i\left(Q_{h+1}^i\right) \right\}_{i=1}^M \in \Theta_h$ that

\begin{align}
\label{formula:best_approximators}
\Delta_h^i \left(Q_{h+1}^i\right)(s, a) = \caT_h^i \left(Q_{h+1}^i\right)(s, a) - \boldsymbol{\phi}(s, a)^\top \dot{\boldsymbol{\theta}}_{h}^i\left(Q_{h+1}^i\right)
\end{align}

where the approximation error $\left\|\Delta_h^i \left(Q_{h+1}^i\right)\right\|_{\infty} \leq \caI$ is small for each $i \in [M]$. We also use $\dot{\boldsymbol{B}}_h \dot{\boldsymbol{w}}_h^i \left(Q_{h+1}^i\right)$ in place of $\dot{\boldsymbol{\theta}}_{h}^i\left(Q_{h+1}^i\right)$ in the following sections since we can write $\dot{\boldsymbol{\theta}}_{h}^i$ as $\dot{\boldsymbol{B}}_h \dot{\boldsymbol{w}}_h^i$ according to Assumption \ref{assumptions:low_rank_small_ibe}.

In the multi-task low-rank least-square regression (equation \ref{formula:linear_rl_least_square}), we are actually trying to recover $\dot{\bm{\theta}}_h^i$. However, due to the noise and representation error (i.e. the inherent Bellman error), we can only obtain an approximate solution $\hat{\bm{\theta}}_h^i = \hat{\bm{B}}_h \hat{\bm{w}}_h^i$ (see the global optimization problem in Definition \ref{formula:linear_rl_global_optimization}). 

\begin{align}
\left(\hat{\boldsymbol{\theta}}_h^1, ..., \hat{\boldsymbol{\theta}}_h^M\right) & = \hat{\boldsymbol{B}}_h \begin{bmatrix} \hat{\boldsymbol{w}}_h^1 & \hat{\boldsymbol{w}}_h^2 & \cdots & \hat{\boldsymbol{w}}_h^M
\end{bmatrix}\\
 & = \argmin_{\left\|\boldsymbol{B}_h \boldsymbol{w}_h^i\right\|_2 \leq D} \sum_{i=1}^M \sum_{j=1}^{t-1} \left(\boldsymbol{\phi}\left(s^i_{h j}, a^i_{hj}\right)^{\top} \boldsymbol{B}_h \boldsymbol{w}_h^i -R\left(s_{h j}^i, a_{h j}^i\right)-\max_a Q_{h+1}^i\left(s_{h+1, j}^i\right)\right)^{2} \\
 \label{formula:linear_rl_low_rank_regression}
 & = \argmin_{\left\|\boldsymbol{B}_h \boldsymbol{w}_h^i\right\|_2 \leq D} \sum_{i=1}^M \sum_{j=1}^{t-1} \left(\boldsymbol{\phi}\left(s^i_{h j}, a^i_{hj}\right)^{\top} \boldsymbol{B}_h \boldsymbol{w}_h^i - \caT_h^i \left(Q_{h+1}^i\right)\left(s^i_{h j}, a^i_{hj}\right) - z_{hj}^i\left(Q_{h+1}^i\right)\left(s^i_{h j}, a^i_{hj}\right)\right)^{2}
\end{align}

where $z_{hj}^i\left(Q_{h+1}^i\right)\left(s^i_{h j}, a^i_{hj}\right) \defeq R\left(s_{h j}^i, a_{h j}^i\right) + \max_a Q_{h+1}^i\left(s_{h+1,j}^i, a\right) - \caT_h^i\left(Q_{h+1}^i\right)\left(s^i_{h j}, a^i_{hj}\right)$.

Define $\boldsymbol{\Phi}_{ht}^i \in \dbR^{(t-1) \times d}$ to be the collection of linear features up to episode $t - 1$ in task $i$, i.e. the $j$-th row of $\boldsymbol{\Phi}_{ht}^i$ is $\boldsymbol{\phi}\left(s^i_{h j}, a^i_{hj}\right)^\top$. Let $\boldsymbol{Y}_{ht}^i \in \dbR^{t-1}$ be a vector whose $j$-th dimension is $\caT_h^i \left(Q_{h+1}^i\right)\left(s^i_{h j}, a^i_{hj}\right) + z_{hj}^i\left(Q_{h+1}^i\right)\left(s^i_{h j}, a^i_{hj}\right)$. Then the objective in (\ref{formula:linear_rl_low_rank_regression}) can be written as

\begin{align}
\argmin_{\left\|\boldsymbol{B}_h \boldsymbol{w}_h^i\right\|_2 \leq D} \sum_{i=1}^M \left\|\boldsymbol{\Phi}_{ht}^i \boldsymbol{B}_h \boldsymbol{w}_h^i - \boldsymbol{Y}_{ht}^i \right\|_2^2
\end{align}

Therefore, we have

\begin{align}
\sum_{i=1}^M \left\|\boldsymbol{\Phi}_{ht}^i \hat{\boldsymbol{B}}_h \hat{\boldsymbol{w}}_h^i\left(Q_{h+1}^i\right) - Y_{ht}^i \right\|_2^2 \leq \sum_{i=1}^M \left\|\boldsymbol{\Phi}_{ht}^i \dot{\boldsymbol{B}}_h \dot{\boldsymbol{w}}_h^i\left(Q_{h+1}^i\right) - \boldsymbol{Y}_{ht}^i \right\|_2^2
\end{align}

which implies

\begin{align}
& \sum_{i=1}^M \left\|\boldsymbol{\Phi}_{ht}^i \hat{\boldsymbol{B}}_h \hat{\boldsymbol{w}}_h^i\left(Q_{h+1}^i\right) - \boldsymbol{\Phi}_{ht}^i \dot{\boldsymbol{B}}_h \dot{\boldsymbol{w}}_h^i\left(Q_{h+1}^i\right) \right\|_2^2 \\ 
\label{formula:linear_rl_noise_error}
& \leq 2 \sum_{i=1}^M \left(\bm{\Delta}_{ht}^i\right)^\top \boldsymbol{\Phi}_{ht}^i \left(\hat{\boldsymbol{B}}_h \hat{\boldsymbol{w}}_h^i\left(Q_{h+1}^i\right) - \dot{\boldsymbol{B}}_h \dot{\boldsymbol{w}}_h^i\left(Q_{h+1}^i\right)\right) \\
\label{formula:linear_rl_projection_error}
& + 2 \sum_{i=1}^M \left(\bm{z}_{ht}^i\right)^\top \boldsymbol{\Phi}_{ht}^i \left(\hat{\boldsymbol{B}}_h \hat{\boldsymbol{w}}_h^i\left(Q_{h+1}^i\right) - \dot{\boldsymbol{B}}_h \dot{\boldsymbol{w}}_h^i\left(Q_{h+1}^i\right)\right)
\end{align}

where $\bm{\Delta}_{ht}^i \defeq \begin{bmatrix} \Delta_{h1}^i\left(Q_{h+1}^i\right)\left(s^i_{h 1}, a^i_{h1}\right) & \Delta_{h2}^i\left(Q_{h+1}^i\right)\left(s^i_{h 2}, a^i_{h2}\right) & \cdots & \Delta_{h, t - 1}^i\left(Q_{h+1}^i\right)\left(s^i_{h, t - 1}, a^i_{h, t-1}\right) \end{bmatrix} \in \dbR^{t-1}$, and $\bm{z}_{ht}^i \defeq \begin{bmatrix} z_{h1}^i\left(Q_{h+1}^i\right)\left(s^i_{h 1}, a^i_{h1}\right) & \cdots & z_{h, t - 1}^i\left(Q_{h+1}^i\right)\left(s^i_{h, t - 1}, a^i_{h, t-1}\right) \end{bmatrix} \in \dbR^{t-1}$.

In the next sections we will show how to bound \ref{formula:linear_rl_noise_error} and \ref{formula:linear_rl_projection_error}.

\newpage

\subsubsection{Failure Event}
\label{appendix_linear_rl:failure_event}

Define the failure event at step $h$ in episode $t$ as 

\begin{definition}[Failure Event]
\label{definitions:failure_event_ht}
\begin{align}
E_{ht} & \defeq I\Big[\exists \left\{Q_{h+1}^i\right\}_{i=1}^M \in \caQ_{h+1} \quad \sum_{i=1}^M \left(\bm{z}_{ht}^i\right)^\top \boldsymbol{\Phi}_{ht}^i \left(\hat{\boldsymbol{B}}_h \hat{\boldsymbol{w}}_h^i\left(Q_{h+1}^i\right) - \dot{\boldsymbol{B}}_h \dot{\boldsymbol{w}}_h^i\left(Q_{h+1}^i\right)\right) > \\
& F_h^1 \sqrt{\sum_{i=1}^{M}\left\|\hat{\boldsymbol{B}}_h \hat{\boldsymbol{w}}_h^i\left(Q_{h+1}^i\right) - \dot{\boldsymbol{B}}_h \dot{\boldsymbol{w}}_h^i\left(Q_{h+1}^i\right)\right\|^{2}_{\tilde{\boldsymbol{V}}_{ht}^i(\lambda)}} + F_h^2 \Big]
\end{align}
\end{definition}

where $F_h^1$ and $F_h^2$ will be specified later.

We have the following lemma to bound the probability of $E_{ht}$.

\begin{lemma}
\label{lemma:linear_rl_failure_event_ht}
For the input parameter $\delta > 0$, there exists $F_h^1$ and $F_h^2$ such that
\begin{align}
\dbP\left(\bigcup_{t=1}^T \bigcup_{h=1}^H E_{ht}\right) \leq \frac{\delta}{2}
\end{align}
\end{lemma}

\begin{proof}
According to Lemma A.5 of \citet{du2020few}, there exists an $\epsilon$-net $\caE_{h+1}^o$ over $\caO^{d \times k}$ (with regards to the Frobenius norm) such that $\left|\caE_{h+1}^o\right| \leq (6\sqrt{k}/\epsilon^\prime)^{kd} $. Moreover, there exists an $\epsilon$-net $\caE^b_{h+1}$ over $\caB^{k}$ that $\left|\caE^b_{h+1}\right| \leq (1 + 2 / \epsilon^\prime)^{k}$. We can show a corresponding $\epsilon$-net $\caE^{\text{mul}}_{h+1} \defeq \caE_{h+1}^o \times \left(\caE^b_{h+1}\right)^M$ over $\Theta_{h+1}$.

For any $\left(Q_{h+1}^1\left(\boldsymbol{B}_{h+1}\boldsymbol{w}^1_{h+1}\right), \cdots, Q_{h+1}^M\left(\boldsymbol{B}_{h+1}\boldsymbol{w}_{h+1}^M\right)\right) \in \caQ_{h+1}$, there exists $\bar{\boldsymbol{B}}_{h+1} \in \caE_{h+1}^o$ and $\left(\bar{\boldsymbol{w}}^1_{h+1}, \cdots, \bar{\boldsymbol{w}}_{h+1}^M\right) \in \left(\caE^b_{h+1}\right)^M$ such that

$$\left\|\boldsymbol{B}_{h+1} - \bar{\boldsymbol{B}}_{h+1}\right\|_F \leq \epsilon^\prime  \quad \left\|\boldsymbol{w}_{h+1}^i - \bar{\boldsymbol{w}}_{h+1}^i\right\|_2 \leq \epsilon^\prime, \forall i \in [M]$$

Therefore, 

$$\left\|\boldsymbol{B}_{h+1} \boldsymbol{w}^i_{h+1} - \bar{\boldsymbol{B}}_{h+1} \bar{\boldsymbol{w}}^i_{h+1}\right\|_2 \leq 2\epsilon^\prime, \forall i \in [M]$$

Define $\bar{Q}_{h+1}^i$ to be $Q_{h+1}^i\left(\bar{\boldsymbol{B}}_{h+1} \bar{\boldsymbol{w}}^i_{h+1}\right)$, and let $\bm{\bar{z}}_{ht}^i \defeq \begin{bmatrix} z_{h1}^i\left(\bar{Q}_{h+1}^i\right)\left(s^i_{h 1}, a^i_{h1}\right) & \cdots & z_{h, t - 1}^i\left(\bar{Q}_{h+1}^i\right)\left(s^i_{h, t - 1}, a^i_{h, t-1}\right) \end{bmatrix} \in \dbR^{t-1}$, then

\begin{align}
& \sum_{i=1}^M \left(\bm{z}_{ht}^i\right)^\top \boldsymbol{\Phi}_{ht}^i \left(\hat{\boldsymbol{B}}_h \hat{\boldsymbol{w}}_h^i\left(Q_{h+1}^i\right) - \dot{\boldsymbol{B}}_h \dot{\boldsymbol{w}}_h^i\left(Q_{h+1}^i\right)\right) \\
& = \sum_{i=1}^M \left(\bm{\bar{z}}_{ht}^i\right)^\top \boldsymbol{\Phi}_{ht}^i \left(\hat{\boldsymbol{B}}_h \hat{\boldsymbol{w}}_h^i\left(Q_{h+1}^i\right) - \dot{\boldsymbol{B}}_h \dot{\boldsymbol{w}}_h^i\left(Q_{h+1}^i\right)\right) \\
& + \sum_{i=1}^M \left(\bm{z}_{ht}^i - \bm{\bar{z}}_{ht}^i\right)^\top \boldsymbol{\Phi}_{ht}^i \left(\hat{\boldsymbol{B}}_h \hat{\boldsymbol{w}}_h^i\left(Q_{h+1}^i\right) - \dot{\boldsymbol{B}}_h \dot{\boldsymbol{w}}_h^i\left(Q_{h+1}^i\right)\right)
\end{align}

For fixed $\left\{\bar{\boldsymbol{B}}_{h+1} \bar{\boldsymbol{w}}^i_{h+1}\right\}_{i=1}^M \in \caE^{\text{mul}}_{h+1}$, $z_{h, j}^i\left(\bar{Q}_{h+1}^i\right)\left(s^i_{h,j}, a^i_{h, j}\right)$ is zero-mean 1-subgaussian conditioned on $\caF_{h, j}$ according to Assumption \ref{assumptions:linear_rl_regularity}. Thus, we can use exactly the same argument as in Lemma \ref{lemma: confidence set for linear bandits} to show that

\begin{align}
& \sum_{i=1}^M \left(\bm{\bar{z}}_{ht}^i\right)^\top \boldsymbol{\Phi}_{ht}^i \left(\hat{\boldsymbol{B}}_h \hat{\boldsymbol{w}}_h^i\left(Q_{h+1}^i\right) - \dot{\boldsymbol{B}}_h \dot{\boldsymbol{w}}_h^i\left(Q_{h+1}^i\right)\right)  \\
& \leq \sqrt{Mk + 5kd \log(kMT) + 2\log(1 / \delta^{\prime})} \sqrt{\sum_{i=1}^{M}\left\|\hat{\boldsymbol{B}}_h \hat{\boldsymbol{w}}_h^i\left(Q_{h+1}^i\right) - \dot{\boldsymbol{B}}_h \dot{\boldsymbol{w}}_h^i\left(Q_{h+1}^i\right)\right\|^{2}_{\tilde{\boldsymbol{V}}_{ht}^i(\lambda)}} \\
&+ \sqrt{2\log(2MT/\delta^{\prime})} + \sqrt{k + 3kd\log(kMT) + \log(1 / \delta^{\prime})}
\end{align}

by setting $\epsilon = \frac{1}{kM^2T^2}$, $\delta_1 = \frac{\delta^{\prime}}{2 \left(\frac{6\sqrt{2k}}{\epsilon}\right)^{2kd}}$, and $\delta_2 = \frac{\delta^{\prime}}{2MT}$ in equation \ref{formula:linear_rl_cross_reference}. Thus, we have that with probability $1 - \delta^{\prime}$ the inequality above holds for any $h \in [H], t \in [T]$. Take $\delta = \frac{\delta^{\prime}}{2\left|\caE^{\text{mul}}_{h+1}\right|}$, by union bound we know the above ineqaulity holds with probability $1 - \delta$ for any $\left\{\bar{\boldsymbol{B}}_{h+1} \bar{\boldsymbol{w}}^i_{h+1}\right\}_{i=1}^M \in \caE^{\text{mul}}_{h+1}$ and any $h \in [H], t \in [T]$.


Since it holds that $\left|Q_{h+1}^i\left(\boldsymbol{B}_{h+1}\boldsymbol{w}_{h+1}^i\right)(s, a) - Q_{h+1}^i\left(\bar{\boldsymbol{B}}_{h+1}\bar{\boldsymbol{w}}_{h+1}^i\right)(s, a)\right| \leq 2\epsilon^\prime$ for any $(s, a) \in \caS \times \caA, i \in [M]$, we have

\begin{align}
\left|z_{hj}^i\left(\bar{Q}_{h+1}^i\right)\left(s^i_{h j}, a^i_{hj}\right) - z_{hj}^i\left(Q_{h+1}^i\right)\left(s^i_{h j}, a^i_{hj}\right)\right| \leq 8\epsilon^\prime
\end{align}

Then we have

\begin{align}
& \sum_{i=1}^M \left(\bm{z}_{ht}^i - \bm{\bar{z}}_{ht}^i\right)^\top \boldsymbol{\Phi}_{ht}^i \left(\hat{B}_h \hat{w}_h^i\left(Q_{h+1}^i\right) - \dot{\boldsymbol{B}}_h \dot{\boldsymbol{w}}_h^i\left(Q_{h+1}^i\right)\right) \\
& \leq \sum_{i=1}^M \left\|\left(\boldsymbol{\Phi}_{ht}^i\right)^\top \left(\bm{z}_{ht}^i - \bm{\bar{z}}_{ht}^i\right)\right\|_{\tilde{\boldsymbol{V}}_{ht}^i(\lambda)^{-1}} \left\|\hat{\boldsymbol{B}}_h \hat{\boldsymbol{w}}_h^i\left(Q_{h+1}^i\right) - \dot{\boldsymbol{B}}_h \dot{\boldsymbol{w}}_h^i\left(Q_{h+1}^i\right)\right\|_{\tilde{\boldsymbol{V}}_{ht}^i(\lambda)} \\
& \leq 8\epsilon^\prime \sqrt{T} \sum_{i=1}^M \left\|\hat{\boldsymbol{B}}_h \hat{\boldsymbol{w}}_h^i\left(Q_{h+1}^i\right) - \dot{\boldsymbol{B}}_h \dot{\boldsymbol{w}}_h^i\left(Q_{h+1}^i\right)\right\|_{\tilde{\boldsymbol{V}}_{ht}^i(\lambda)} \\
& \leq 8\epsilon^\prime \sqrt{MT} \sqrt{\sum_{i=1}^{M}\left\|\hat{\boldsymbol{B}}_h \hat{\boldsymbol{w}}_h^i\left(Q_{h+1}^i\right) - \dot{\boldsymbol{B}}_h \dot{\boldsymbol{w}}_h^i\left(Q_{h+1}^i\right)\right\|^{2}_{\tilde{\boldsymbol{V}}_{ht}^i(\lambda)}}
\end{align}

for arbitrary $\{Q_{h+1}^i\}$ and any $h \in [H], t \in [T]$. The second inequality follows from the Projection Bound (Lemma 8) in \citet{zanette2020learning}. 

Take $\epsilon^\prime = 1 / 8\sqrt{MT}$, we finally finish the proof by setting 

\begin{align}
F_h^1 & \defeq \sqrt{9kd \log(kMT) + 5Mk \log(MT) + 2\log(2 / \delta)} \\
F_h^2 & \defeq \sqrt{4kd \log(kMT) + 5Mk \log(MT) + 2\log(2 / \delta)} \\
& + \sqrt{k + 5kd\log(kMT) + 2Mk \log(MT) + \log(2 / \delta)}
\end{align}

\end{proof}

In the next sections we assume the failure event $\bigcup_{t=1}^T \bigcup_{h=1}^H E_{ht}$ won't happen.

\newpage

\subsubsection{Bellman Error}
\label{appendix_linear_rl:bellman_error}

Outside the failure event, we can bound the estimation error of the least-square regression \ref{formula:linear_rl_least_square}.

\begin{lemma}
\label{lemma:linear_rl_least_square_error}
For any episode $t \in [T]$ and step $h \in [H]$, any $\left\{Q_{h+1}^i\right\}_{i=1}^M \in \caQ_{h+1}$, we have

\begin{align}
\sum_{i=1}^{M}\left\|\hat{\boldsymbol{B}}_h \hat{\boldsymbol{w}}_h^i\left(Q_{h+1}^i\right) - \dot{\boldsymbol{B}}_h \dot{\boldsymbol{w}}_h^i\left(Q_{h+1}^i\right)\right\|^{2}_{\tilde{\boldsymbol{V}}_{ht}^i(\lambda)} \leq \alpha_{ht} \defeq \left(2\sqrt{MT} \caI + 2 F_{h}^1 + \sqrt{2F_h^2 + 4MD^2 \lambda}\right)^2
\end{align}
\end{lemma}

\begin{proof}
Recall that

\begin{align}
& \sum_{i=1}^M \left\|\boldsymbol{\Phi}_{ht}^i \hat{\boldsymbol{B}}_h \hat{\boldsymbol{w}}_h^i\left(Q_{h+1}^i\right) - \boldsymbol{\Phi}_{ht}^i \dot{\boldsymbol{B}}_h \dot{\boldsymbol{w}}_h^i\left(Q_{h+1}^i\right) \right\|_2^2 \\ 
& \leq 2 \sum_{i=1}^M \left(\bm{\Delta}_{ht}^i\right)^\top \boldsymbol{\Phi}_{ht}^i \left(\hat{\boldsymbol{B}}_h \hat{\boldsymbol{w}}_h^i\left(Q_{h+1}^i\right) - \dot{\boldsymbol{B}}_h \dot{\boldsymbol{w}}_h^i\left(Q_{h+1}^i\right)\right) \\
& + 2 \sum_{i=1}^M \left(\bm{z}_{ht}^i\right)^\top \boldsymbol{\Phi}_{ht}^i \left(\hat{\boldsymbol{B}}_h \hat{\boldsymbol{w}}_h^i\left(Q_{h+1}^i\right) - \dot{\boldsymbol{B}}_h \dot{\boldsymbol{w}}_h^i\left(Q_{h+1}^i\right)\right)
\end{align}

For the first term, we have

\begin{align}
& \sum_{i=1}^M \left(\bm{\Delta}_{ht}^i\right)^\top \Phi_{ht}^i \left(\hat{\boldsymbol{B}}_h \hat{\boldsymbol{w}}_h^i\left(Q_{h+1}^i\right) - \dot{\boldsymbol{B}}_h \dot{\boldsymbol{w}}_h^i\left(Q_{h+1}^i\right)\right) \\
& \leq \sum_{i=1}^M \left\|\left(\boldsymbol{\Phi}_{ht}^i\right)^\top \bm{\Delta}_{ht}^i\right\|_{\tilde{\boldsymbol{V}}_{ht}^i(\lambda)^{-1}} \left\|\hat{\boldsymbol{B}}_h \hat{\boldsymbol{w}}_h^i\left(Q_{h+1}^i\right) - \dot{\boldsymbol{B}}_h \dot{\boldsymbol{w}}_h^i\left(Q_{h+1}^i\right)\right\|_{\tilde{\boldsymbol{V}}_{ht}^i(\lambda)} \\
& \leq \sqrt{T} \caI \sum_{i=1}^M \left\|\hat{\boldsymbol{B}}_h \hat{\boldsymbol{w}}_h^i\left(Q_{h+1}^i\right) - \dot{\boldsymbol{B}}_h \dot{\boldsymbol{w}}_h^i\left(Q_{h+1}^i\right)\right\|_{\tilde{\boldsymbol{V}}_{ht}^i(\lambda)} \\
& \leq \sqrt{MT} \caI \sqrt{\sum_{i=1}^{M}\left\|\hat{\boldsymbol{B}}_h \hat{\boldsymbol{w}}_h^i\left(Q_{h+1}^i\right) - \dot{\boldsymbol{B}}_h \dot{\boldsymbol{w}}_h^i\left(Q_{h+1}^i\right)\right\|^{2}_{\tilde{\boldsymbol{V}}_{ht}^i(\lambda)}}
\end{align}

The second inequality follows from the Projection Bound (Lemma 8) in \citet{zanette2020learning}, and the last inequality is due to Cauchy-Schwarz.

Outside the failure event, we have

\begin{align}
& \sum_{i=1}^{M}\left\|\hat{\boldsymbol{B}}_h \hat{\boldsymbol{w}}_h^i\left(Q_{h+1}^i\right) - \dot{\boldsymbol{B}}_h \dot{\boldsymbol{w}}_h^i\left(Q_{h+1}^i\right)\right\|^{2}_{\tilde{\boldsymbol{V}}_{ht}^i(\lambda)} \\
& \leq \sum_{i=1}^M \left\|\boldsymbol{\Phi}_{ht}^i \hat{\boldsymbol{B}}_h \hat{\boldsymbol{w}}_h^i\left(Q_{h+1}^i\right) - \boldsymbol{\Phi}_{ht}^i \dot{\boldsymbol{B}}_h \dot{\boldsymbol{w}}_h^i\left(Q_{h+1}^i\right) \right\|_2^2 + 4MD^2 \lambda \\
& \leq \left(2\sqrt{MT} \caI + 2 F_{h}^1\right) \sqrt{\sum_{i=1}^{M}\left\|\hat{\boldsymbol{B}}_h \hat{\boldsymbol{w}}_h^i\left(Q_{h+1}^i\right) - \dot{\boldsymbol{B}}_h \dot{\boldsymbol{w}}_h^i\left(Q_{h+1}^i\right)\right\|^{2}_{\tilde{\boldsymbol{V}}_{ht}^i(\lambda)}} + 2F_{h}^2 + 4MD^2 \lambda 
\end{align}

which implies

\begin{align}
& \sum_{i=1}^{M}\left\|\hat{\boldsymbol{B}}_h \hat{\boldsymbol{w}}_h^i\left(Q_{h+1}^i\right) - \dot{\boldsymbol{B}}_h \dot{\boldsymbol{w}}_h^i\left(Q_{h+1}^i\right)\right\|^{2}_{\tilde{\boldsymbol{V}}_{ht}^i(\lambda)} \\
& \leq \left(2\sqrt{MT} \caI + 2 F_{h}^1\right)^2 + 2F_h^2 + 4MD^2 \lambda + \left(2\sqrt{MT} \caI + 2 F_{h}^1\right) \sqrt{2F_h^2 + 4MD^2 \lambda}\\    
& \leq \left(2\sqrt{MT} \caI + 2 F_{h}^1 + \sqrt{2F_h^2 + 4MD^2 \lambda}\right)^2
\end{align}
\end{proof}

\begin{lemma}[Bound on Bellman Error]
\label{lemma:linear_rl_bellman_error}
Outside the failure event, for any feasible solution $\left\{Q_h^i\left(\bar{\theta}_{h}^i\right)\right\}_{h}^i $ ($\bar{Q}_{h}^i$ for short, with a little abuse of notations) of the global optimization procedure in definition \ref{formula:linear_rl_global_optimization}, for any $(s, a) \in \caS \times \caA$, any $h \in [H]$, $t \in [T]$

\begin{align}
\sum_{i=1}^M \left|\bar{Q}_{h}^i (s, a) - \caT_h^i \bar{Q}_{h+1}^i (s, a)\right| \leq  M \caI + 2 \sqrt{\alpha_{ht} \cdot \sum_{i=1}^M \left\|
\boldsymbol{\phi}(s, a)\right\|_{\tilde{\boldsymbol{V}}_{ht}^i(\lambda)^{-1}}^2} 
\end{align}

\end{lemma}

\begin{proof}
\begin{align}
\sum_{i=1}^M \left|\bar{Q}_{h}^i (s, a) - \caT_h^i \bar{Q}_{h+1}^i (s, a)\right| & = \sum_{i=1}^M \left|\boldsymbol{\phi}(s, a)^\top \bar{\boldsymbol{\theta}}_{h}^i - \boldsymbol{\phi}(s, a)^\top \dot{\boldsymbol{\theta}}_{h}^i\left(\bar{Q}_{h+1}^i\right) - \Delta_h^i \left(\bar{Q}_{h+1}^i\right)(s, a)\right| \\
& \leq M\caI + \sum_{i=1}^M \left|\boldsymbol{\phi}(s, a)^\top \bar{\boldsymbol{\theta}}_{h}^i - \boldsymbol{\phi}(s, a)^\top \dot{\boldsymbol{\theta}}_{h}^i\left(\bar{Q}_{h+1}^i\right)\right|\\
& \leq M\caI + \sum_{i=1}^M \left(\left|\boldsymbol{\phi}(s, a)^\top \dot{\boldsymbol{\theta}}_h^i\left(\bar{Q}_{h+1}^i\right) - \boldsymbol{\phi}(s, a)^\top \hat{\boldsymbol{\theta}}_h^i\right| + \left|\boldsymbol{\phi}(s, a)^\top \hat{\boldsymbol{\theta}}_h^i - \boldsymbol{\phi}(s, a)^\top \bar{\boldsymbol{\theta}}_h^i\right|\right) \\
& \leq M\caI + \sum_{i=1}^M \left\|
\boldsymbol{\phi}(s, a)\right\|_{\tilde{\boldsymbol{V}}_{ht}^i(\lambda)^{-1}} \left(\left\|\dot{\boldsymbol{\theta}}_h^i\left(\bar{Q}_{h+1}^i\right) - \hat{\boldsymbol{\theta}}_h^i\right\|_{\tilde{\boldsymbol{V}}_{ht}^i(\lambda)} + \left\|\hat{\boldsymbol{\theta}}_h^i - \bar{\boldsymbol{\theta}}_h^i\right\|_{\tilde{\boldsymbol{V}}_{ht}^i(\lambda)}\right) \\
& \leq M \caI + 2 \sqrt{\alpha_{ht} \cdot \sum_{i=1}^M \left\|
\boldsymbol{\phi}(s, a)\right\|_{\tilde{\boldsymbol{V}}_{ht}^i(\lambda)^{-1}}^2} 
\end{align}

The first equality is due to the definition of $\Delta_h^i \left(\bar{Q}_{h+1}^i\right)(s, a)$. The last inequality is due to lemma \ref{lemma:linear_rl_least_square_error}.
\end{proof}

\newpage

\subsubsection{Optimism}
\label{appendix_linear_rl:optimism}

We can find the "best" approximator of optimal value functions in our function class recursively defined as

\begin{align}
\label{formula:linear_rl_best_approximator}
\left(\boldsymbol{\theta}_h^{1*}, \boldsymbol{\theta}_h^{2*}, \cdots, \boldsymbol{\theta}_h^{M*}\right) \defeq \argmin_{\left(\boldsymbol{\theta}_h^1, \boldsymbol{\theta}_h^2, \cdots, \boldsymbol{\theta}_h^M\right) \in \Theta_h} \sup_{s, a, i} \left|\left(\boldsymbol{\phi}(s, a)^\top \boldsymbol{\theta}_h^i - \caT_h^i Q_{h+1}^i \left(\boldsymbol{\theta}_{h+1}^{i*}\right)\right)(s, a)\right|
\end{align}

with $\boldsymbol{\theta}_{H+1}^{i*}=\boldsymbol{0}, \forall i \in [M]$

For the accuracy of this best approximator, we have

\begin{lemma}
\label{lemma:linear_rl_accuracy_best_approximator}
For any $h \in [H]$, 
$$
\sup _{(s, a) \in \caS \times \caA, i \in [M]}\left|Q_{h}^{i*}(s, a)-\boldsymbol{\phi}(s, a)^{\top} \boldsymbol{\theta}_{h}^{*}\right| \leq (H-h+1) \mathcal{I}
$$
\end{lemma}

where $Q_h^{i*}$ is the optimal value function for task $i$. This lemma is derived directly from Lemma 6 in \citet{zanette2020learning}.

For our solution of the problem in Definition \ref{formula:linear_rl_global_optimization} in episode $t$, we have the following lemma:

\begin{lemma}
\label{lemma:linear_rl_optimism}
$\left\{\left(\boldsymbol{\theta}_h^{1*}, \boldsymbol{\theta}_h^{2*}, \cdots, \boldsymbol{\theta}_h^{M*}\right)\right\}_{h=1}^H$ is a feasible solution of the problem in Definition \ref{formula:linear_rl_global_optimization}. Moreover, denote the solution of the problem in Definition \ref{formula:linear_rl_global_optimization} in episode $t$ by $\bar{\boldsymbol{\theta}}_{ht}^i$ for $h \in [H], i \in [M]$, it holds that

\begin{align}
\sum_{i=1}^M V_{1}^i\left(\bar{\boldsymbol{\theta}}_{1 t}^i\right)\left(s_{1t}^i\right) \geq \sum_{i=1}^M V_{1}^{i*}\left(s_{1t}^i\right) - MH\caI  
\end{align}
\end{lemma}

\begin{proof}
First we show that $\left\{\left(\boldsymbol{\theta}_h^{1*}, \boldsymbol{\theta}_h^{2*}, \cdots, \boldsymbol{\theta}_h^{M*}\right)\right\}_{h=1}^H$ is a feasible solution. We can construct $\left\{\bar{\boldsymbol{\xi}}_h^i\right\}_{i=1}^M$ so that $\bar{\boldsymbol{\theta}}_h^i = \boldsymbol{\theta}_h^{i*}$ and no other constraints are violated. We use an inductive construction, and the base case when $\bar{\bm{\theta}}_{H+1}^i = \bm{\theta}_{H+1}^{i*} = 0$ is trivial.

Now suppose we have $\left\{\bar{\boldsymbol{\xi}}_y^i\right\}_{i=1}^M$ for $y = h + 1, ..., H$ such that $\bar{\bm{\theta}}_{y}^i = \bm{\theta}_{y}^{i*}$ for $y = h + 1, ..., H$ and $i \in [M]$, we show we can find $\left\{\bar{\boldsymbol{\xi}}_h^i\right\}_{i=1}^M$ so $\bar{\boldsymbol{\theta}}_{h}^i = \boldsymbol{\theta}_{h}^{i*}$ for $i \in [M]$, and no constraints are violated. From the definition of $\boldsymbol{\theta}_{h}^{i*}$ we can set (with a little abuse of notations)

\begin{align}
\dot{\boldsymbol{\theta}}_h^i\left(\boldsymbol{\theta}_{h+1}^{i*}\right) = \boldsymbol{\theta}_{h}^{i*}
\end{align}

According to lemma \ref{lemma:linear_rl_least_square_error} we have

\begin{align}
\sum_{i=1}^{M}\left\|\hat{\boldsymbol{\theta}}_h^i\left(\boldsymbol{\theta}_{h+1}^{i*}\right) - \dot{\boldsymbol{\theta}}_h^i\left(\boldsymbol{\theta}_{h+1}^{i*}\right)\right\|^{2}_{\tilde{\boldsymbol{V}}_{ht}^i(\lambda)} \leq \alpha_{ht}    
\end{align}

Therefore, set $\bar{\boldsymbol{\xi}}_h^i = \dot{\boldsymbol{\theta}}_h^i\left(\boldsymbol{\theta}_{h+1}^{i*}\right) -  \hat{\boldsymbol{\theta}}_h^i\left(\boldsymbol{\theta}_{h+1}^{i*}\right)$, then

\begin{align}
\bar{\boldsymbol{\theta}}_h^i & = \hat{\boldsymbol{\theta}}_h^i\left(\bar{\boldsymbol{\theta}}_{h+1}^{i}\right) + \bar{\boldsymbol{\xi}}_h^i \\
& = \hat{\boldsymbol{\theta}}_h^i\left(\boldsymbol{\theta}_{h+1}^{i*}\right) + \dot{\boldsymbol{\theta}}_h^i\left(\boldsymbol{\theta}_{h+1}^{i*}\right) -  \hat{\boldsymbol{\theta}}_h^i\left(\boldsymbol{\theta}_{h+1}^{i*}\right) \\
& = \boldsymbol{\theta}_h^{i*}
\end{align}

Finally, we can verify $\left(\bar{\boldsymbol{\theta}}_h^1,...,\bar{\boldsymbol{\theta}}_h^M\right) \in \Theta_h$ from $\left(\boldsymbol{\theta}_h^{1*}, \cdots, \boldsymbol{\theta}_h^{M*}\right) \in \Theta_h$.

Since $\bar{\boldsymbol{\theta}}_{1t}^i$ is the optimal solution, we can finish the proof by showing

\begin{align}
\sum_{i=1}^M V_{1}^i\left(\bar{\boldsymbol{\theta}}_{1 t}^i\right)\left(s_{1t}^i\right) & = \sum_{i=1}^M \max_a \boldsymbol{\phi}\left(s_{1t}^i, a\right)^\top \bar{\boldsymbol{\theta}}_{1 t}^i \\
& \geq \sum_{i=1}^M \max_a \boldsymbol{\phi}\left(s_{1t}^i, a\right)^\top \boldsymbol{\theta}_{1}^{i*} \qquad \text{(since $\theta_{1}^{i*}$ is the feasible solution)} \\
& \geq \sum_{i=1}^M \boldsymbol{\phi}\left(s_{1t}^i, \pi_1^{i*}\left(s_{1t}^i\right)\right)^\top \boldsymbol{\theta}_{1}^{i*} \\
& \geq \sum_{i=1}^M Q_h^{i*}\left(s_{1t}^i, \pi_1^{i*}\left(s_{1t}^i\right)\right) - MH \caI \qquad \text{(by Lemma \ref{lemma:linear_rl_accuracy_best_approximator})}\\
& \geq \sum_{i=1}^M V_h^{i*}\left(s_{1t}^i\right) - MH \caI
\end{align}

\end{proof}

\newpage

\subsubsection{Regret Bound}
\label{appendix_linear_rl:regret_bound}

We are ready to present the proof of our regret bound.

From Lemma \ref{lemma:linear_rl_failure_event_ht} we know that the failure event $\bigcup_{t=1}^T \bigcup_{h=1}^H E_{ht}$ happens with probability at most $\delta / 2$, so we assume it does not happen. Then we can decompose the regret as

\begin{align}
\text{Reg}(T) & = \sum_{t=1}^T \sum_{i=1}^M \left(V_1^{i*} - V_1^{\pi_{t}^i} \right)\left(s_{1t}^i\right) \\
& = \sum_{t=1}^T \sum_{i=1}^M \left(V_1^{i*} - V_{1}^i\left(\bar{\boldsymbol{\theta}}_{1 t}^i\right) \right)\left(s_{1t}^i\right) + \sum_{t=1}^T \sum_{i=1}^M \left(V_{1}^i\left(\bar{\boldsymbol{\theta}}_{1 t}^i\right) - V_1^{\pi_{t}^i} \right)\left(s_{1t}^i\right) \\
& \leq \sum_{t=1}^T \sum_{i=1}^M \left(V_{1}^i\left(\bar{\boldsymbol{\theta}}_{1 t}^i\right) - V_1^{\pi_{t}^i} \right)\left(s_{1t}^i\right) + MHT \caI \qquad \text{(by Lemma \ref{lemma:linear_rl_optimism})}
\end{align}

Let $a_{ht}^i = \pi_{t}^i\left(s_{ht}^i\right)$, and denote $Q_{h}^i\left(\bar{\boldsymbol{\theta}}_{ht}^i\right)$($V_{h}^i\left(\bar{\boldsymbol{\theta}}_{ht}^i\right)$) by $\bar{Q}_{ht}^i$($\bar{V}_{ht}^i$) for short, we have

\begin{align}
\sum_{i=1}^M \left(\bar{V}_{ht}^i - V_h^{\pi_{t}^i} \right)\left(s_{ht}^i\right) & = \sum_{i=1}^M \left(\bar{Q}_{ht}^i - Q_h^{\pi_{t}^i} \right)\left(s_{ht}^i, a_{ht}^i\right) \\
& = \sum_{i=1}^M \left(\bar{Q}_{ht}^i - \caT_h^i \bar{Q}_{h+1,t}^i \right)\left(s_{ht}^i, a_{ht}^i\right) + \sum_{i=1}^M \left(\caT_h^i \bar{Q}_{h+1,t}^i - Q_h^{\pi_{t}^i} \right)\left(s_{ht}^i, a_{ht}^i\right) \\
& \leq M\caI + 2 \sqrt{\alpha_{ht} \cdot \sum_{i=1}^M \left\|
\boldsymbol{\phi}\left(s_{ht}^i, a_{ht}^i\right)\right\|_{\tilde{\boldsymbol{V}}_{ht}^i(\lambda)^{-1}}^2} + \sum_{i=1}^M \dbE_{s^\prime \sim p_h^i\left(s_{ht}^i, a_{ht}^i\right)}\left[\left(\bar{V}_{h+1,t}^i - V_{h+1}^{\pi_{t}^i} \right)\left(s^\prime\right)\right] \\
\label{formula:linear_rl_recursive_regret}
& \leq \sum_{i=1}^M \left(\bar{V}_{h+1,t}^i - V_{h+1}^{\pi_{t}^i} \right)\left(s_{h+1,t}^i\right) + M\caI + 2 \sqrt{\alpha_{ht} \cdot \sum_{i=1}^M \left\|
\boldsymbol{\phi}\left(s_{ht}^i, a_{ht}^i\right)\right\|_{\tilde{\boldsymbol{V}}_{ht}^i(\lambda)^{-1}}^2} + \sum_{i=1}^M \zeta_{ht}^i
\end{align}

where $\zeta_{ht}^i$ is a martingale difference with regards to the filtration $\caF_{h,t}$ defined as

\begin{align}
\zeta_{ht}^i \defeq \left(\bar{V}_{h+1,t}^i - V_{h+1}^{\pi_{t}^i} \right)\left(s_{h+1,t}^i\right) - \dbE_{s^\prime \sim p_h^i\left(s_{ht}^i, a_{ht}^i\right)}\left[\left(\bar{V}_{h+1,t}^i - V_{h+1}^{\pi_{t}^i} \right)\left(s^\prime\right)\right]
\end{align}

According to assumption \ref{assumptions:linear_rl_regularity} we know $\left|\zeta_{ht}^i\right| \leq 4$, so we can apply Azuma-Hoeffding's inequality that with probability $1 - \delta/2$ for any $t \in [T]$ and $i \in [M]$

\begin{align}
\label{formula:linear_rl_martingale}
\sum_{j=1}^t \zeta_{ht}^i \leq 4\sqrt{2t \ln\left(\frac{2T}{\delta}\right)}
\end{align}

By applying inequality \ref{formula:linear_rl_recursive_regret} recursively, we can bound the regret as

\begin{align}
\text{Reg}(T) & \leq \sum_{t=1}^T \sum_{i=1}^M \left(\bar{V}_{1t}^i - V_1^{\pi_{t}^i} \right)\left(s_{1t}^i\right) + MHT \caI \\
& \leq 2MHT\caI + \sum_{t=1}^T \sum_{h=1}^H 2 \sqrt{\alpha_{ht} \cdot \sum_{i=1}^M \left\|
\boldsymbol{\phi}\left(s_{ht}^i, a_{ht}^i\right)\right\|_{\tilde{\boldsymbol{V}}_{ht}^i(\lambda)^{-1}}^2} + \sum_{i=1}^M \sum_{h=1}^H \sum_{t=1}^T \zeta_{ht}^i 
\end{align}

The last inequality is due to $\bar{V}_{H+1}^i(s) = \max_a \boldsymbol{\phi}(s, a)^\top \bar{\boldsymbol{\theta}}_{H+1,t}^i = 0, V_{H+1}^{\pi_t^i}(s) = 0$.

The Lemma 11 of \citet{abbasi2011improved} gives that for any $i \in [M]$ and $h \in [H]$

\begin{align}
\sum_{t=1}^T \left\|
\boldsymbol{\phi}\left(s_{ht}^i, a_{ht}^i\right)\right\|_{\tilde{\boldsymbol{V}}_{ht}^i(\lambda)^{-1}}^2 = \tilde{O}\left(d\right)
\end{align}

Moreover, by the definition of $\alpha_{ht}$ (see Lemma \ref{lemma:linear_rl_least_square_error}) we know that for any $h \in [H]$ and $t \in [T]$

\begin{align}
\alpha_{ht} = \tilde{O}\left(Mk + kd + MT\caI^2\right)
\end{align}

Take all of above we can show the final regret bound.

\begin{align}
\text{Reg}(T) & \leq 2MHT\caI + \sum_{t=1}^T \sum_{h=1}^H 2 \sqrt{\alpha_{ht} \cdot \sum_{i=1}^M \left\|
\boldsymbol{\phi}\left(s_{ht}^i, a_{ht}^i\right)\right\|_{\tilde{\boldsymbol{V}}_{ht}^i(\lambda)^{-1}}^2} + \sum_{i=1}^M \sum_{h=1}^H \sum_{t=1}^T \zeta_{ht}^i \\
& = \tilde{O}\left(MHT\caI + \tilde{O}\left(\sqrt{Mk + kd + MT\caI^2}\right) \sum_{h=1}^H \sum_{t=1}^T \sqrt{\sum_{i=1}^M \left\|
\boldsymbol{\phi}\left(s_{ht}^i, a_{ht}^i\right)\right\|_{\tilde{\boldsymbol{V}}_{ht}^i(\lambda)^{-1}}^2} + MH\sqrt{T}\right) \\
& = \tilde{O}\left(MHT\caI + \tilde{O}\left(\sqrt{Mk + kd + MT\caI^2}\right) \sum_{h=1}^H \sqrt{T} \cdot \sqrt{\sum_{t=1}^T \sum_{i=1}^M \left\|
\boldsymbol{\phi}\left(s_{ht}^i, a_{ht}^i\right)\right\|_{\tilde{\boldsymbol{V}}_{ht}^i(\lambda)^{-1}}^2} + MH\sqrt{T}\right) \\
& = \tilde{O}\left(MHT\caI + \tilde{O}\left(\tilde{O}\left(\sqrt{Mk + kd + MT\caI^2}\right) \cdot H\sqrt{MTd}\right)  + MH\sqrt{T}\right) \\
& = \tilde{O}\left(HM\sqrt{dkT} + Hd\sqrt{MkT} + HMT\sqrt{d}\caI\right)
\end{align}

\newpage

\subsection{Proof of Theorem~\ref{theorem:linear_rl_lower_bound}}
\label{sec: proof of the lower bound for rl}

To prove the lower bound for multi-task RL, our idea is to connect the  lower bound for the multi-task learning problem to the lower bound in the single-task LSVI setting~\citep{zanette2020learning}. in the paper of~\citet{zanette2020learning}, they assumed the feature dimension $d$ can be varied among different steps, which is denoted as $d_h$ for step $h$. They proved the lower bound for linear RL in this setting is $\Omega\left(\sum_{h=1}^{H}d_h \sqrt{T} +\sum_{h=1}^{H} \sqrt{d_h}\mathcal{I}T\right)$. However, this lower bound is derived by the hard instance with $d_1 = \sum_{h=2}^{H} d_h$. If we set $d_1 = d_2 = \cdots = d_H = d$ like our setting, we can only obtain the lower bound of $\Omega\left(d \sqrt{T} + \sqrt{d}\mathcal{I}T\right)$ following their proof idea. In fact, the dependence on $H$ in this lower bound can be further improved. In order to obtain a tighter lower bound, we consider the lower bound for single-task misspecified linear MDP. This setting can be proved to be strictly simpler than the LSVI setting following the idea of Proposition 3 in~\citet{zanette2020learning}. The lower bound for misspecified linear MDP can thus be applied to LSVI setting.

\newpage

\subsubsection{Lower Bounds for single-task RL}

This subsection focus on the lower bound for misspecifed linear MDP setting, in which the transition kernel and the reward function are assume to be approximately linear.

\begin{assumption}
\label{assumption: misspecified linear MDP}
(Assumption B in~\citet{jin2020provably}) For any $\zeta \leq 1$, we say that $\operatorname{MDP}(\mathcal{S}, \mathcal{A},p,r,H)$ is a $\zeta$-approximate linear MDP with a feature map $\boldsymbol{\phi}: \mathcal{S} \times \mathcal{A} \rightarrow \mathbb{R}^d$, if for any $h \in [H]$, there exist $d$ unknown measures $\boldsymbol{\theta}_h = (\theta_h^{(1)},\cdots, \theta_h^{(d)})$ over $\mathcal{S}$ and an unknown vector $\boldsymbol{\nu}_h \in \mathbb{R}^d$ such that for any $(s,a) \in \mathcal{S}\times \mathcal{A}$, we have
\begin{align}
    & \|p_h(\cdot|s,a) - \left\langle\boldsymbol{\phi}(s,a),\boldsymbol{\theta}_h(\cdot)\right\rangle\|_{\operatorname{TV}} \leq \zeta \\
    & |r_h(s,a) - \left\langle\boldsymbol{\phi}(s,a),\boldsymbol{\nu}_h\right\rangle| \leq \zeta
\end{align}
\end{assumption}

For regularity, we assume that Assumption~\ref{assumptions:linear_rl_regularity} still holds, and we also assume that there exists a constant $D$ such that $\|\boldsymbol{\theta}_h(s)\| \leq D$ for all $s \in \mathcal{S}, h \in [H]$,  $\|\boldsymbol{\nu}_h\| \leq D$ for all $h \in [H]$. $D \geq 4$ suffices in our hard instance construction.

For misspecifed linear MDP, we can prove the following lower bound.

\begin{proposition}
\label{thm: lower bound for linear MDP}
Suppose $T \geq \frac{d^2H}{4}$, $d \geq 10$, $H \geq 10$ and $\zeta \leq \frac{1}{4H}$, there exist a $\zeta$-approximate linear MDP class such that the expected regret of any algorithm on at least a member of the MDP class is at least $\Omega\left(d\sqrt{HT} + HT\mathcal{I} \sqrt{d}\right)$.
\end{proposition}

To prove the lower bound, our basic idea is to connect the problem to $\frac{H}{2}$ linear bandit problems. Similar hard instance construction has been used in~\citet{zhou2020nearly,zhou2020provably}. In our construction, the state space $\mathcal{S}$ consists of $H+2$ states, which is denoted as $x_1, x_2, \cdots, x_{H+2}$. The agent starts the episode in state $x_1$. In $x_h$, it can either transits to $x_{h+1}$ or $x_{H+2}$ with certain transition probability. If the agent enters $x_{H+2}$, it will stay in this state in the remaining steps, i.e. $x_{H+2}$ is an absorbing state. For each state, there are $2^{d-4}$ actions and $\mathcal{A} = \{-1,1\}^{d-4}$. Suppose the agent takes action $\boldsymbol{a} \in \{-1,1\}^{d-4}$ in state $s_h$, the transition probability to state $s_{h+1}$ and $s_{H+2}$ is $1-\zeta_h(\boldsymbol{a}) -\delta-\boldsymbol{\mu}_h^{\top}\boldsymbol{a}$ and $\delta+ \zeta_h(\boldsymbol{a})+\boldsymbol{\mu}_h^{\top}\boldsymbol{a}$ respectively. Here  $|\zeta_h(\boldsymbol{a})| \leq \zeta$ denotes the approximation error of linear representation, $\delta= 1/H$ and $\boldsymbol{\mu}_h \in \{-\Delta, \Delta\}^{d-4}$ with $\Delta = \sqrt{\delta/T}/(4\sqrt{2}) $ so that the probability is well-defined. The reward can only be obtained in $x_{H+2}$, with $r_h(x_{H+2,a}) = 1/H$ for any $h,a$. We assume the reward to be deterministic.

We can check that this construction satisfies Assumption~\ref{assumption: misspecified linear MDP} with $\boldsymbol{\phi}$ and $\boldsymbol{\theta}$ defined in the following way:

$$ \boldsymbol{\phi}(s,\boldsymbol{a})=\left\{
\begin{aligned}
&\left(0, \alpha, \alpha\delta,0, \beta \boldsymbol{a}^{\top}\right)^{\top} & \quad s = x_1, x_2, \cdots, x_H\\
&\left(0, 0,0, \alpha, \boldsymbol{0}^{\top} \right)^{\top} & \quad s = x_{H+1}\\
&\left(\alpha, 0, 0,\alpha, \boldsymbol{0}^{\top}\right)^{\top} & \quad s = x_{H+2}
\end{aligned}
\right.
$$

$$ \boldsymbol{\theta}_h(s')=\left\{
\begin{aligned}
&\left(0,\frac{1}{\alpha}, -\frac{1}{\alpha},0, -\frac{\boldsymbol{\mu}^{\top}_h}{\beta}\right)^{\top} & \quad s' = x_{h+1}\\
&\left(0, 0, \frac{1}{\alpha},\frac{1}{\alpha}, \frac{\boldsymbol{\mu}^{\top}_h}{\beta} \right)^{\top} & \quad s = x_{H+2}\\
&\boldsymbol{0} & \quad \operatorname{otherwise}
\end{aligned}
\right.
$$

$\boldsymbol{\nu}_h$ is defined to be $(\frac{1}{H\alpha},\boldsymbol{0}^{\top})^{\top}$, and $\alpha = \sqrt{1/(2+\Delta(d-4))}$, $\beta = \sqrt{\Delta/(2+\Delta(d-4))}$. Note that $\|\boldsymbol{\phi}(s,a)\| \leq 1$, $\|\boldsymbol{\theta}_h(s')\| \leq D$ and $\|\boldsymbol{\nu}_h\| \leq D$ hold for any $s,a,s',h$ when $T \geq d^2H/4 $.

Since the rewarding state is only $x_{H+2}$, the optimal strategy in state $x_{h}$ ($h \leq H$) is to take an action that maximizes the probability of entering $x_{H+2}$, i.e., to maximize $\boldsymbol{\mu}_h^{\top} \boldsymbol{a} + \zeta(\boldsymbol{a})$. That is to say, we can regard the problem of finding the optimal action in state $s_h$ and step $h$ as finding the optimal arm for a $d-4$-dimensional approximately (misspecified) linear bandits problem. Thanks to the choice of $\delta$ such that $(1-\delta)^{H/2}$ is a constant, there is sufficiently high probability of entering state $x_h$ for any $h \leq H/2$. Therefore, we can show that this problem is  harder than solving $H/2$ misspecified linear bandit problems. This following lemma characterizes this intuition. The lemma follows the same idea of Lemma~C.7 in~\citet{zhou2020nearly}, though our setting is more difficult since we consider misspecified case.
\begin{lemma}
\label{lemma: value decomposition for linear MDP lower bound}
Suppose $H \geq 10$, $d \geq 10$ and $(d-4)\Delta \leq \frac{1}{2H}$. We define $r^{b}_h(\boldsymbol{a}) = \boldsymbol{\mu}^{\top}\boldsymbol{a} + \zeta_h(\boldsymbol{a}) $, which can be regarded as the corresponding reward for the equivalent linear bandit problem in step $h$. Fix $\boldsymbol{\mu} \in (\{-\Delta,\Delta\}^{d-4})^{H}$. Fix a possibly history dependent policy $\pi$. Letting $V^{\star}$ and $V^{\pi}$ be the optimal value function and the value function of policy $\pi$ respectively, we have
\begin{align}
    V_1^{\star}(s_1) - V^{\pi}_1(s_1) \geq 0.02 \sum_{h=1}^{H/2} \left(\max_{\boldsymbol{a} \in \mathcal{A}} r^{b}_h(\boldsymbol{a}) - \sum_{\boldsymbol{a} \in \mathcal{A}} \pi_h(\boldsymbol{a}|s_h) r^{b}_h(\boldsymbol{a})\right)
\end{align}
\end{lemma}
\begin{proof}
Note that the only rewarding state is $x_{H+2}$ with $r_h(x_{H+2},\boldsymbol{a}) = \frac{1}{H}$. Therefore, the value function of a certain policy $\pi$ can be calculated as:
\begin{align}
    V_1^{\pi}(x_1) = \sum_{h=1}^{H-1} \frac{H-h}{H} \mathbb{P}(N_h|\pi)
\end{align}
where $N_h$ denotes the event of visiting state $x_h$ in step $h$ and then transits to $x_{H+2}$, i.e. $N_h = \{s_{h} = x_h, s_{h+1} = x_{H+2}\}$. Suppose $\omega^{\pi}_h = \sum_{\boldsymbol{a} \in \mathcal{A}} \pi_h(\boldsymbol{a}|s_h) r^{b}_h(\boldsymbol{a})$ and $\omega^{\star}_h = \max_{\boldsymbol{a} \in \mathcal{A}} r^{b}_h(\boldsymbol{a})$. By the law of total probability and the Markov property, we have
\begin{align}
    \mathbb{P}(N_h|\pi) = (\delta+\omega^{\pi}_h) \prod_{j=1}^{h-1} (1-\delta - \omega^{\pi}_h)
\end{align}

Thus we have
\begin{align}
    V_1^{\pi}(x_1) = \sum_{h=1}^{H-1} \frac{H-h}{H} (\delta+\omega^{\pi}_h) \prod_{j=1}^{h-1} (1-\delta - \omega^{\pi}_h)
\end{align}

Similarly, for the value function of the optimal policy, we have
\begin{align}
    V_1^{\star}(x_1) = \sum_{h=1}^{H-1} \frac{H-h}{H} (\delta+\omega^{\star}_h) \prod_{j=1}^{h-1} (1-\delta - \omega^{\star}_h)
\end{align}

Define $S_i = \sum_{h=i}^{H-1}\frac{H-h}{H} (\delta+\omega^{\pi}_h)\prod_{j=i}^{h-1}(1-\delta - \omega_h^{\pi})$ and $T_i = \sum_{h=i}^{H-1}\frac{H-h}{H} (\delta+\omega^{\star}_h)\prod_{j=i}^{h-1}(1-\delta - \omega_h^{\star})$. Then we have $V^{\star}_1(x_1) - V^{\pi}_1(x_1) = T_1 - S_1$. Notice that
\begin{align}
    S_i &= \frac{H-i}{H} (\omega^{\pi}_i + \delta) + S_{i+1} (1-\omega_i^{\pi} -\delta) \\
    T_i &= \frac{H-i}{H} (\omega^{\star}_i + \delta) + T_{i+1} (1-\omega_i^{\star} -\delta) 
\end{align}

Thus we have
\begin{align}
    T_i - S_i = \left(\frac{H-i}{H} - T_{i+1}\right)\left(\omega^{\star}_i - \omega^{\pi}_i\right) + (T_{i+1} - S_{i+1})(1-\omega^{\pi}_i - \delta)
\end{align}

By induction, we get 
\begin{align}
    \label{eqn: Vstar - Vpi lower bound}
    T_1 - S_1 = \sum_{h=1}^{H-1} (\omega^{\star}_i - \omega^{\pi}_i)(\frac{H-h}{H}-T_{h+1}) \prod_{j=1}^{h-1} (1-\omega_j^{\pi}-\delta)
\end{align}

Since the reward is non-negative and only occurs in $x_{H+2}$, we know that $V^{\star}_1(x_1) \geq V^{\star}_{2}(x_2) \geq \cdots \geq V^{\star}_1(x_H)$. Thus we have $T_h \leq T_1 = V^{\star}_1(x_1) \leq  \sum_{h=1}^{H} \mathbb{P}(N_h|\pi^{\star})$. If $N_h$ doesn't happen for any $h \in [H]$, then the agent must enter $x_{H+1}$. The probability of this event has the following form:
\begin{align}
   \mathbb{P}\left(\neg \left(\cup_{h\in [H]}N_h|\pi^{\star}\right)\right)   = & 1- \prod_{h=1}^{H} \mathbb{P}(N_h|\pi^{\star}) \\
    = &  \prod_{h\in[H]}\left(1-\delta- \omega^{\star}_{h}\right) \\
   \geq &  \prod_{h\in [H]} (1-\frac{1}{H} + \frac{1}{2H}) \\
   = & (1-\frac{1}{2H})^{H} \\
   \geq & 0.6
\end{align}
The fist inequality is due to $\delta = \frac{2}{H}$ and $|\omega_h^{\star}| \leq \frac{1}{H}$. The above discussion indicates that $T_h \leq 0.4$, thus $\frac{H-h}{H}-T_{h+1} \geq 0.1$ for $h \leq H/2$. Similarly, $\prod_{j=1}^{h-1} (1-\omega_j^{\pi}-\delta) \geq (1-\frac{3}{2H})^{H-1} \geq 0.2$. Combining with Eqn~\ref{eqn: Vstar - Vpi lower bound}, we have
\begin{align}
    T_1 -S_1 \geq 0.02 \sum_{h=1}^{\frac{H}{2}}  (\omega_h^{\star} - \omega_h^{\pi}) = 0.02  \sum_{h=1}^{H/2} \left(\max_{\boldsymbol{a} \in \mathcal{A}} r^{b}_h(\boldsymbol{a}) - \sum_{\boldsymbol{a} \in \mathcal{A}} \pi_h(\boldsymbol{a}|s_h) r^{b}_h(\boldsymbol{a})\right)
\end{align}

Combining with the definition of $T_1$ and $S_1$, we can prove the lemma.
\end{proof}

After proving Lemma~\ref{lemma: value decomposition for linear MDP lower bound}, we are ready to prove Proposition~\ref{thm: lower bound for linear MDP}.
\begin{proof}
(proof of Proposition~\ref{thm: lower bound for linear MDP}) By Lemma~\ref{lemma: value decomposition for linear MDP lower bound}, we know that we can decompose the sub-optimality gap of a policy $\pi$ in the following way:
\begin{align}
    V_1^{\star}(s_1) - V^{\pi}_1(s_1) \geq 0.02 \sum_{h=1}^{H/2} \left(\max_{\boldsymbol{a} \in \mathcal{A}} r^{b}_h(\boldsymbol{a}) - \sum_{\boldsymbol{a} \in \mathcal{A}} \pi_h(\boldsymbol{a}|s_h) r^{b}_h(\boldsymbol{a})\right)
\end{align}
where $r^{b}_h(\boldsymbol{a}) = \boldsymbol{\mu}^{\top}\boldsymbol{a} + \zeta_h(\boldsymbol{a}) $, which can be regarded as a reward function for misspecified linear bandit. To prove Theorem~\ref{thm: lower bound for linear MDP}, the only remaining problem is to derive the lower bound for misspecified linear bandits. We directly apply the following two lower bounds for linear bandits.
\begin{lemma}
\label{lemma: lower bound linear bandit for RL 1}
(Lemma~C.8 in~\citet{zhou2020nearly}) Fix a positive real $0 < \delta \leq 1/3$, and positive integers $T,d$ and assume that $T \geq d^2/(2\delta)$ and consider the linear bandit problem $\mathcal{L}_{\boldsymbol{\mu}}$ parametrized with a parameter vector $\boldsymbol{\mu} \in \{-\Delta,\Delta\}^{d}$ and action set $\mathcal{A} = \{-1,1\}^{d}$ so that the reward distribution for taking action $\boldsymbol{a} \in \mathcal{A}$ is a Bernoulli distribution $B(\delta+(\boldsymbol{\mu}^{\star})^{\top}\boldsymbol{a})$. Then for any bandit algorithm $\mathcal{B}$, there exists a $\mu^{*} \in \{-\Delta, \Delta\}^d$ such that the expected pseudo-regret of $\mathcal{B}$ over $T$ steps on bandit $\mathcal{L}_{\boldsymbol{\mu}^{\star}}$ is lower bounded by $\frac{d\sqrt{T\delta}}{8\sqrt{2}}$.
\end{lemma}

\begin{lemma}
\label{lemma: lower bound linear bandit for RL 2} (Proposition 6 in~\citet{zanette2020learning}) There exists a feature map $\phi:\mathcal{A} \rightarrow \mathbb{R}^d$ that defines a misspecified linear bandits class $\mathcal{M}$ such that every bandit instance in that class has reward response:
\begin{align*}
    \mu_a = \phi_a^{\top} \theta + z_a
\end{align*}
for any action $a$ (Here $z_a \in [0,\zeta]$ is the deviation from linearity and $\mu_a \in [0,1]$) and such that the expected regret of any algorithm on at least a member of the class up to round $T$ is $\Omega(\sqrt{d}\zeta T)$.
\end{lemma}

Lemma~\ref{lemma: lower bound linear bandit for RL 1} is used to prove the lower bound for linear mixture MDPs in~\citet{zhou2020nearly}, which states that the lower bound for linear bandits with approximation error $\zeta = 0$, while Lemma~\ref{lemma: lower bound linear bandit for RL 2} mainly consider the influence of $\zeta$ to the lower bound. Combining these two lemmas, the regret lower bound for misspecifid linear bandit is $\Omega(\max(d\sqrt{T\delta},\sqrt{d}\zeta T)) = \Omega (d\sqrt{T\delta} + \sqrt{d}\zeta T)$. Since here our problem can reduce from $H/2$ misspecified linear bandit, we know that the regret lower bound is $\Omega (Hd\sqrt{T\delta} + H\sqrt{d}\zeta T) = \Omega (d\sqrt{HT} + H\sqrt{d}\zeta T)$
\end{proof}

Now we obtain the regret lower bound for misspecified linear MDP. We can prove the corresponding lower bound for the LSVI setting ~\citet{zanette2020learning} since LSVI setting is strictly harder than linear MDP setting. The following lemma states this relation between two settings.
\begin{lemma}
\label{lemma: reduction from linear MDP to LSVI}
If an MDP$(\mathcal{S},\mathcal{A}, p,r, H)$ is a misspecifed linear MDP with approximation error $\zeta$, then this MDP satisfies the low inherent Bellman error assumption with $\mathcal{I} = 2\zeta$.
\end{lemma}
\begin{proof}
If an MDP is an $\zeta$-approximate linear MDP, then we have 
\begin{align}
    & \|p_h(\cdot|s,a) - \left\langle\boldsymbol{\phi}(s,a),\boldsymbol{\theta}_h(\cdot)\right\rangle\|_{\operatorname{TV}} \leq \zeta \\
    & |r_h(s,a) - \left\langle\boldsymbol{\phi}(s,a),\boldsymbol{\nu}_h\right\rangle| \leq \zeta
\end{align}

For any $\theta_{h+1} \in \mathbb{R}^d$, we have $\mathcal{T}_{h}\left(Q_{h+1}(\theta_{h+1})\right)(s, a) = r_{h}(s, a)+\mathbb{E}_{s^{\prime} \sim p_{h}(\cdot \mid s, a)} V_{h+1}(\theta_{h+1})\left(s^{\prime}\right)$. Since $V_{h+1}(\theta_{h+1})\left(s^{\prime}\right) \leq 1$, plugging the approximately linear form of $r_h(s,a)$ and $p_h(\cdot|s,a)$, we have
\begin{align}
    |\mathcal{T}_{h}\left(Q_{h+1}(\theta_{h+1})\right)(s, a) - \left \langle \boldsymbol{\phi}(s,a), \sum_{s'} \boldsymbol{\theta}_h(s')V_{h+1}(\theta_{h+1})\left(s^{\prime}\right) + \boldsymbol{\nu}_h  \right\rangle| \leq 2\zeta
\end{align}
\end{proof}

By lemma~\ref{lemma: reduction from linear MDP to LSVI}, we can directly apply the hard instance construction and the lower bound for misspecified linear MDP to LSVI setting.

\begin{proposition}
\label{proposition: referred lower bound for RL}
 There exist function feature maps $\boldsymbol{\phi}_1,...,\boldsymbol{\phi}_H$ that define an MDP class $\mathcal{M}$ such that every MDP in that class satisfies low inherent Bellman error at most $\mathcal{I}$ and such that the expected reward on at least a member of the class (for $|\mathcal{A}| \geq 3, d,k,H \geq 10, T= \Omega(d^2H),\mathcal{I} \leq \frac{1}{4H}$) is $\Omega (d\sqrt{HT} + d H\mathcal{I}T)$.
\end{proposition}

\newpage

\subsubsection{Lower Bound for Multi-task RL}

In order to prove Theorem~\ref{theorem:linear_rl_lower_bound}, we need to prove and then combine the following two lemmas.

\begin{lemma}
\label{lemma: rl lower bound lemma 1}
Under the setting of Theorem~\ref{theorem:linear_rl_lower_bound}, the expected regret of any algorithm $\mathcal{A}$ is lower bounded by $\Omega( Mk\sqrt{HT})$.
\end{lemma}

\begin{lemma}
\label{lemma: rl lower bound lemma 2}
Under the setting of Theorem~\ref{theorem:linear_rl_lower_bound}, the expected regret of any algorithm $\mathcal{A}$ is lower bounded by $\Omega\left(d\sqrt{kMHT} + HMT\sqrt{d}\mathcal{I}\right)$.
\end{lemma}


These two lemmas are proved by reduction from Proposition~\ref{proposition: referred lower bound for RL}, which is a lower bound we proved for the single-task LSVI setting.

\begin{proof}
(Proof of Lemma~\ref{lemma: rl lower bound lemma 1}) The lemma is proved by contradiction. Suppose there is an algorithm $\mathcal{A}$ that achieves $\sup_{M \in \mathcal{M}} \mathbb{E}[Reg(T)] \leq C  M k\sqrt{HT}$ for a constant $C$. Then there must exists a task $i \in [M]$, such that the expected regret for this single task is at most $C  k\sqrt{HT}$. However, by Proposition~\ref{proposition: referred lower bound for RL}, the expected regret for MDPs with dimension $k$ in horizon $h$ is at least $\Omega (k\sqrt{HT} + \sqrt{k}H \mathcal{I}T)$. This leads to a contradiction.
\end{proof}

\begin{proof}
(Proof of Lemma~\ref{lemma: rl lower bound lemma 2}) The hard instance construction follows the same idea of the proof for our Lemma~\ref{lemma: lower bound for linear bandits, approximation error term}, as well as the hard instance to prove Lemma~19 in \citet{yang2020provable}. Without loss of generality, we assume that $M$ can be exactly divided by $k$. 

We divide $M$ tasks into $k$ groups. Each group shares the same parameter $\{\boldsymbol{\theta}^i_h\}_{h=1}^{H}$. To be more specific, we let $\boldsymbol{w}_h^{1} = \boldsymbol{w}_h^{2} = \cdots = \boldsymbol{w}_h^{M/k} = \boldsymbol{e}_h^1$, $\boldsymbol{w}_h^{M/k+1} = \boldsymbol{w}_h^{M/k+2} = \cdots = \boldsymbol{w}_h^{2M/k} = \boldsymbol{e}_h^2$, $\cdots$, $\boldsymbol{w}_h^{(k-1)M/k+1} = \boldsymbol{w}_h^{(k-1)M/k+2} = \cdots = \boldsymbol{w}_h^{M} = \boldsymbol{e}_h^k$. Under this construction, the parameters $\boldsymbol{\theta}_h^i$ for these tasks are exactly the same in each group, but relatively independent among different groups. That is to say, the expected regret lower bound is at least the summation of the regret lower bounds in all $k$ groups.

Now we consider the regret lower bound for group $j \in [k]$. Since the parameters are shared in the same group, the regret of running an algorithm for $M/k$ tasks with $T$ episodes each is at least the regret of running an algorithm for single-task linear bandit with $M/k \cdot T$ episodes. By Proposition~\ref{proposition: referred lower bound for RL}, the regret for single-task linear bandit with $MT/k$ episodes is at least $\Omega (d\sqrt{MHT/k} +  \sqrt{d} \mathcal{I}HMT/k)$. Summing over all $k$ groups, we can prove that the regret lower bound is $\Omega (d\sqrt{kHMT} +  \sqrt{d} \mathcal{I}HMT)$.
\end{proof}

\end{document}